\documentclass{article}

\PassOptionsToPackage{numbers}{natbib}

\usepackage[final]{neurips_2020}

\usepackage[utf8]{inputenc} 
\usepackage[T1]{fontenc}    
\usepackage{hyperref}       
\usepackage{url}            
\usepackage{booktabs}       
\usepackage{amsfonts}       
\usepackage{nicefrac}       
\usepackage{microtype}      

\title{Minimax Value Interval for Off-Policy Evaluation and Policy Optimization}

\author{%
  Nan Jiang \\
  Department of Computer Science\\
  University of Illinois at Urbana-Champaign\\
  Urbana, IL 61801 \\
  \texttt{nanjiang@illinois.edu} \\
  \And
  Jiawei Huang \\
  Department of Computer Science\\
  University of Illinois at Urbana-Champaign\\
  Urbana, IL 61801 \\
  \texttt{jiaweih@illinois.edu}
}

\usepackage{subcaption}
\usepackage{url, hyperref}
\usepackage{algorithm, algpseudocode} %

\usepackage{amsthm, amsmath, amssymb} 
\usepackage[shortlabels]{enumitem}
\usepackage{comment}
\usepackage{color}
\usepackage{multirow}
\usepackage{xspace}
\usepackage{rotating}
\usepackage{graphicx}
\usepackage{caption}
\usepackage{xcolor}
\usepackage[normalem]{ulem}

\theoremstyle{plain}
\newtheorem{theorem}{\textbf{Theorem}}
\newtheorem{lemma}[theorem]{\textbf{Lemma}}

\newtheorem{proposition}[theorem]{Proposition}

\theoremstyle{definition}

\newtheorem{assumption}{Assumption}

\newtheorem{remark}{Remark}

\renewcommand{\cite}{\citep}

\newcommand{\tabincell}[2]{\begin{tabular}{@{}#1@{}}#2\end{tabular}}

\newcommand{\EE}{\mathbb{E}}

\newcommand{\Indi}{\mathbb{I}}

\DeclareMathOperator*{\argmin}{arg\,min}
\DeclareMathOperator*{\argmax}{arg\,max}

\newcommand{\Fcal}{\mathcal{F}}
\newcommand{\Qcal}{\mathcal{Q}}
\newcommand{\Wcal}{\mathcal{W}}

\newcommand{\Mcal}{\mathcal{M}}
\newcommand{\Scal}{\mathcal{S}}

\newcommand{\Acal}{\mathcal{A}}
\newcommand{\Tcal}{\mathcal{T}}

\newcommand{\Rmin}{R_{\min}}
\newcommand{\Rmax}{R_{\max}}

\newcommand{\RR}{\mathbb{R}}

\newcommand{\conv}{\mathcal{C}}
\newcommand{\ubw}{\mathrm{UB_w}}
\newcommand{\lbw}{\mathrm{LB_w}}
\newcommand{\ubq}{\mathrm{UB_q}}
\newcommand{\lbq}{\mathrm{LB_q}}
\newcommand{\ubwpi}{\mathrm{UB_w^\pi}}
\newcommand{\lbwpi}{\mathrm{LB_w^\pi}}

\newcommand{\lbqpi}{\mathrm{LB_q^\pi}}
\newcommand{\infw}{\inf_{w\in\Wcal}}
\newcommand{\supw}{\sup_{w\in\Wcal}}
\newcommand{\infq}{\inf_{q\in\Qcal}}
\newcommand{\supq}{\sup_{q\in\Qcal}}

\newcommand{\supCw}{\sup_{w\in\conv(\Wcal)}}

\newcommand{\Lwq}{L(w, q)}
\newcommand{\Lw}[1]{L(w, #1)}
\newcommand{\Lq}[1]{L(#1, q)}
\newcommand{\Lmwl}{L_{\mathrm{w}}}
\newcommand{\Lmql}{L_{\mathrm{q}}}
\newcommand{\wpi}{w_{\pi/\mu}}
\newcommand{\qpi}{Q^{\pi}}
\newcommand{\MRmax}{M_{\max}}
\newcommand{\MRmin}{M_{\min}}
\newcommand{\Sk}{\tilde{\Scal}}

\newcommand{\hq}{\hat{q}}

\newcommand{\lwfq}{``weight-learning''\xspace}
\newcommand{\lqfw}{``value-learning''\xspace}
\newcommand{\Lwfq}{``Weight-learning''\xspace}
\newcommand{\Lqfw}{``Value-learning''\xspace}

\newcommand{\hLwq}{\hat{L}(w,q)}
\newcommand{\hRcal}{\hat{\mathcal{R}}}

\begin{document}

\maketitle

\begin{abstract}
We study minimax methods for off-policy evaluation (OPE) using value functions and marginalized importance weights. Despite that they hold promises of overcoming the exponential variance in traditional importance sampling, several key problems remain: \\
(1) They require function approximation and are generally biased. For the sake of trustworthy OPE, is there anyway to quantify the biases? \\
(2) They are split into two styles (``weight-learning'' vs ``value-learning''). Can we unify them? \\
In this paper we answer both questions positively. By slightly altering the derivation of previous methods (one from each style \citep{uehara2019minimax}), we unify them into a single \emph{value interval} that comes with a special type of double robustness: when \emph{either} the value-function or the importance-weight class is well specified, the interval is valid and its length quantifies the misspecification of the other class. Our interval also provides a unified view of and new insights to some recent methods, and we further explore the implications of our results on exploration and exploitation in off-policy policy optimization with insufficient data coverage.
\end{abstract}

\newcount\Comments  
\Comments=0 
\definecolor{darkgreen}{rgb}{0,0.5,0}
\definecolor{darkred}{rgb}{0.7,0,0}
\definecolor{teal}{rgb}{0.3,0.8,0.8}
\newcommand{\kibitz}[2]{\ifnum\Comments=1\textcolor{#1}{#2}\fi}
\newcommand{\nan}[1]{\kibitz{red}{[NJ: #1]}}

\newcommand{\para}[1]{\noindent \textbf{#1}~}

\section{Introduction} \label{sec:intro}
A major barrier to applying reinforcement learning (RL) to real-world applications is the difficulty of evaluation: how can we reliably evaluate a new policy before actually deploying it, possibly using historical data collected from a different policy? Known as off-policy evaluation (OPE), the problem is genuinely difficult as the variance of any unbiased estimator---including the popular importance sampling methods and their variants \citep{precup2000eligibility, jiang2016doubly}---inevitably grows exponentially in horizon \citep{li2015minimax}. 

To overcome this ``curse of horizon'', the RL community has recently gained interest in a new family of algorithms \citep[e.g.,][]{liu2018breaking}, which require function approximation of value-functions and marginalized importance weights and provide accurate evaluation when \emph{both} function classes are well-specified (or \emph{realizable}). Despite that fast progress is made in this direction, several key problems remain:
\begin{itemize}[leftmargin=*]
	\item The methods are generally biased since they rely on function approximation.  Is there anyway we can quantify the biases, which is important for trustworthy evaluation?
	\item The original method by \citet{liu2018breaking} estimates the marginalized importance weights (``weight'') using a discriminator class of value-functions (``value'').  
	Later, \citet{uehara2019minimax} swap the roles of value and weight to learn a $Q$-function using weight discriminators. Not only we have two styles of methods now (\lwfq vs \lqfw), each of them also ignores some important components of the data in their core optimization (see Sec.~\ref{sec:related} for details). Can we have a unified method that makes effective use of all components of data?
\end{itemize}
In this paper we answer both questions positively. By modifying the derivation of one method from each style \citep{uehara2019minimax}, we unify them into a single \emph{value interval}, which automatically comes with a special type of double robustness: when either the weight- or the value-class is well specified, the interval is valid and its length quantifies the misspecification of the other class (Sec.~\ref{sec:ope}). 
Each bound is computed from a single optimization program that uses all components of the data, which we show is generally tighter than the na\"ive intervals developed from previous methods. 
Our derivation also unifies several recent OPE methods and reveals their simple and direct connections; see Table~\ref{tab:compare} in the appendix. 
Furthermore, we examine the potential of applying these value bounds to two long-standing problems in RL: 
reliable off-policy policy optimization under poor data coverage (i.e., \emph{exploitation}), and efficient \emph{exploration}. Based on a simple but important observation, that poor data coverage can be treated as a special case of importance-weight misspecification, we show that optimizing our lower and upper bounds over a policy class corresponds to the well-established pessimism and optimism principles for these problems, respectively (Sec.~\ref{sec:opt}). 

\para{On Statistical Errors} 
We assume exact expectations and ignore statistical errors 
in most of the derivations, as our main goal is to quantify the biases and unify existing methods. In Appendix~\ref{app:staterr} we show how to theoretically handle the statistical errors by adding generalization error bounds to the interval. That said, these generalization bounds are typically loose for practical purposes, and we handle statistical errors by bootstrapping in the experiments (Section~\ref{sec:ope_exp}) and show its effectiveness empirically. We also refer the readers to concurrent works that provide tighter and/or more efficient computation of confidence intervals for related estimators \cite{feng2020accountable, dai2020coindice}. 

\section{Preliminaries}
\para{Markov Decision Processes}
An infinite-horizon discounted MDP is specified by $(\Scal, \Acal, P, R, \gamma, s_0)$, where $\Scal$ is the state space, $\Acal$ is the action space, $P: \Scal\times\Acal\to\Delta(\Scal)$ is the transition function ($\Delta(\cdot)$ is probability simplex), $R: \Scal\times\Acal\to \Delta([0, \Rmax])$ is the reward function, and $s_0$ is a known and deterministic starting state, which is w.l.o.g.\footnote{When $s_0\sim d_0$ is random, all our derivations hold by replacing $q(s_0, \pi)$ with $\EE_{s_0 \sim d_0}[q(s_0, \pi)]$. \label{ft:initial}} For simplicity we assume $\Scal$ and $\Acal$ are finite and discrete but their cardinalities can be arbitrarily large. Any policy\footnote{Our derivation also applies to stochastic policies.} $\pi: \Scal\to\Acal$ induces a distribution of the trajectory,
$
s_0, a_0, r_0, s_1, a_1, r_1, \ldots, 
$ 
where $s_0$ is the starting state, and $\forall t \ge 0$, $a_t = \pi(s_t)$, $r_t \sim R(s_t, a_t)$, $s_{t+1} \sim P(s_t, a_t)$. The expected discounted return determines the performance of policy $\pi$, which is defined as
$
J(\pi) := \EE[\sum_{t=0}^\infty \gamma^t r_t \,|\, \pi].
$ 
It will be useful to define the discounted (state-action) occupancy of $\pi$ as
$
d^\pi(s,a) := \sum_{t=0}^\infty \gamma^t d_t^\pi(s,a),
$ 
where $d_t^\pi(\cdot,\cdot)$ is the marginal distribution of $(s_t, a_t)$ under policy $\pi$.  $d^\pi$ behaves like an \emph{unnormalized} distribution ($\|d^\pi\|_1 = 1/(1-\gamma)$), and for notational convenience we write $\EE_{d^\pi}[\cdot]$ with the understanding that
$$ \textstyle
\EE_{d^\pi}[f(s,a,r,s')] := \sum_{s\in\Scal,a\in\Acal} d^\pi(s,a) ~\EE_{r\sim R(s,a), s'\sim P(s,a)}[f(s,a,r,s')].
$$
With this notation, the discounted return can be written as $J(\pi) = \EE_{d^\pi}[r]$. 

The policy-specific $Q$-function $Q^\pi$ satisfies the Bellman equations: $\forall s\in\Scal, a\in\Acal$,
$Q^\pi(s,a) = \EE_{r \sim R(s,a), s'\sim P(s,a)}[r + \gamma Q^\pi(s', \pi)] =: (\Tcal^\pi Q^\pi)(s,a),$ 
where $Q^\pi(s', \pi)$ is a shorthand for $Q^\pi(s', \pi(s'))$ and $\Tcal^\pi$ is the Bellman update operator. It will be also useful to keep in mind that
\begin{align} \label{eq:Rq}
J(\pi) = Q^\pi(s_0, \pi).
\end{align} 

\para{Data and Marginalized Importance Weights} \label{sec:data}
In off-policy RL, we are passively given a dataset and cannot interact with the environment to collect more data. The goal of OPE is to estimate $J(\pi)$ for a given policy $\pi$ using the dataset. We assume that the dataset consists of i.i.d.~$(s,a,r,s')$ tuples, where $(s,a) \sim \mu$, $r\sim R(s,a)$, $s'\sim P(s,a)$. $\mu \in \Delta(\Scal\times\Acal)$ is a distribution from which we draw $(s,a)$, which determines the exploratoriness of the dataset. We write $\EE_{\mu}[\cdot]$ as a shorthand for taking expectation w.r.t.~this distribution. The strong i.i.d.~assumption is only meant to simplify derivation and presentation, and does not play a crucial role in our results as we do not handle statistical errors. 

A concept crucial to our discussions is the \emph{marginalized importance weights}. Given any $\pi$, if $\mu(s,a) > 0$ whenever $d^\pi(s,a)>0$, define 
$
\wpi(s,a) := \frac{d^\pi(s,a)}{\mu(s,a)}.
$ 
When there exists $(s,a)$ such that $\mu(s,a)=0$ but $d^\pi(s,a)>0$,  $\wpi$ does not exist (and hence cannot be realized by any function class). When it does exist, $\forall f$,
$ 
\EE_{d^\pi}[f(s,a,r,s')] = \EE_{\wpi}[f(s,a,r,s')] := \EE_{\mu}[\wpi(s,a) f(s,a,r,s')],
$ 
where
\begin{align*} 
\EE_{w}[\cdot] := \EE_{\mu}[w(s,a) \cdot (\cdot)]
\end{align*}
is a shorthand we will use throughout the paper, and $\mu$ is omitted in $\EE_w[\cdot]$ since importance weights are always applied on the data distribution $\mu$. 
Finally, OPE would be easy if we knew $\wpi$, as
\begin{align} \label{eq:Rw}
J(\pi) = \EE_{d^\pi}[r] = \EE_{\wpi}[r].
\end{align}

\para{Function Approximation}
Throughout the paper, we assume access to two function classes $\Qcal \subset (\Scal\times\Acal\to \RR)$ and $\Wcal \subset (\Scal\times\Acal\to \RR)$. To develop intuition, they are supposed to model $\qpi$ and $\wpi$, respectively, though most of our main results are stated without assuming any kind of realizability. We use $\conv(\cdot)$ to denote the convex hull of a set. We also make the following compactness assumption so that infima and suprema we introduce later are always attainable. 
\begin{assumption}
	We assume $\Qcal$ and $\Wcal$ are compact subsets of $\RR^{\Scal\times\Acal}$.
\end{assumption}

\section{Related Work} \label{sec:related}
\para{Minimax OPE} \citet{liu2018breaking} proposed the first minimax algorithm for learning marginalized importance weights. When the data distribution (e.g., our $\mu$) covers $d^\pi$ reasonably well, the method provides efficient estimation of $J(\pi)$ without incurring the exponential variance in horizon, a major drawback of importance sampling \citep{precup2000eligibility, li2015minimax, jiang2016doubly}. Since then, the method has sparked a flurry of interest in the RL community \citep{XieTengyang2019OOEf, ChowYinlam2019DBEo, liu2019understanding, liu2019off, rowland2019conditional, liao2019off, zhang2019gendice, zhang2020gradientdice}.

While most of these methods solve for a weight-function using value-function discriminators and plug into Eq.\eqref{eq:Rw} to form the final estimate of $J(\pi)$, \citet{uehara2019minimax} recently show that one can flip the roles of weight- and value-functions to approximate $Q^\pi$. This is also closely related to the kernel loss for solving Bellman equations \citep{feng2019kernel}. 
As we will see, when we keep model misspecification in mind and derive upper and lower  bounds of $J(\pi)$ (as opposed to point estimates), the two types of methods merge into almost the same value intervals---except that they are the reverse of each other, and at most one of them is valid in general (Sec.~\ref{sec:unify}).

\para{Drawback of Existing Methods} One drawback of both types of methods is that some important components of the data are ignored in the core optimization. For example, \lwfq \citep[e.g.,][]{liu2018breaking, zhang2019gendice} completely ignores the rewards $r$ in its loss function, and only uses it in the final plug-in step. Similarly, \lqfw \citep[MQL of][]{uehara2019minimax} ignores the initial state $s_0$ until the final plug-in step \citep[see also][]{feng2019kernel}. In contrast, each of our bounds is computed from a single optimization program that uses all components of the data, and we show the advantage of unified optimization in Table~\ref{tab:compare} and App.~\ref{app:naiveCI}. 

\para{Double Robustness} Our interval is valid when either function class is well-specified, which can be viewed as a type of double robustness. This is related to but different from the usual notion of double robustness in RL \citep{dudik2011doubly, jiang2016doubly, KallusNathan2019EBtC, tang2019harnessing}, and we explain the difference in App.~\ref{app:dr}.


\para{AlgaeDICE} Closest related to our work is the recently proposed AlgaeDICE for off-policy policy optimization  \citep{nachum2019algaedice}. In fact, one side of our interval recovers a version of its policy evaluation component. \citet{nachum2019algaedice} derive the expression using Fenchel duality \citep[see also][]{nachum2020reinforcement}, whereas we provide an alternative derivation using basic telescoping properties such as Bellman equations (Lemmas~\ref{lem:pel_q} and \ref{lem:pel_w}). Our results also provide further justification for AlgaeDICE as an off-policy policy optimization algorithm (Sec.~\ref{sec:opt}), and point out its weakness for OPE (Sec.~\ref{sec:ope}) and how to address it. 



\section{The Minimax Value Intervals} \label{sec:ope}
In this section we derive the minimax value intervals by slightly altering the derivation of two recent methods \citep{uehara2019minimax}, one of \lwfq style (Sec.~\ref{sec:mwl-ci}) and one of \lqfw style (Sec.~\ref{sec:mql-ci}), and show that under certain conditions, they merge into a single unified value interval whose validity only relies on either $\Qcal$ or $\Wcal$ being well-specified (Sec.~\ref{sec:unify}). While we focus on the discounted \& behavior-agnostic setting in the main text, in Appendix~\ref{app:variants} we describe how to adapt to other settings such as when behavior policy is available or in average-reward MDPs.

\subsection{Value Interval for Well-specified $\Qcal$ and Misspecified $\Wcal$} \label{sec:mwl-ci}
We start with a simple lemma that can be used to derive the \lwfq methods, such as the original algorithm by \citet{liu2018breaking} and its behavior-agnostic extension by \citet{uehara2019minimax}. Our derivation assumes realizable $\Qcal$ but arbitrary $\Wcal$. 

\begin{lemma}[Evaluation Error Lemma for Importance Weights]\label{lem:pel_q}
	For any $w:\Scal\times\Acal \to \RR$, 
	\begin{align} \label{eq:pel_q}
	J(\pi) -  \EE_w[r] = Q^\pi(s_0, \pi) + \EE_w[\gamma Q^\pi(s', \pi) - Q^\pi(s,a)].
	\end{align}
\end{lemma}
\begin{proof}
	By moving terms, it suffices to show that $J(\pi) - Q^\pi(s_0, \pi) = \EE_w[r + \gamma Q^\pi(s', \pi) - Q^\pi(s,a)]$, which holds because both sides equal 0. 
\end{proof}

\para{\Lwfq in a Nutshell} The \lwfq methods aim to learn a $w$ such that $J(\pi) \approx \EE_w[r]$. By Lemma~\ref{lem:pel_q}, we may simply find $w$ that sets the RHS of Eq.\eqref{eq:pel_q} to $0$. Of course, this expression depends on $Q^\pi$ which is unknown. However, if we are given a function class $\Qcal$ that captures $Q^\pi$, we can find $w$ (over a class $\Wcal$) that minimizes
\begin{align} \label{eq:loss_mwl}
\supq |\Lmwl(w, q)|,~~ \textrm{where~~} \Lmwl(w,q) := q(s_0, \pi) + \EE_w[\gamma q(s', \pi) - q(s,a)]. 
\end{align}
This derivation implicitly assumes that we can find $w$ such that $\Lmwl(w, q) \approx 0~ \forall q \in \Qcal$, which is guaranteed when $\wpi \in \Wcal$. When $\Wcal$ is misspecified, however, the estimate can be highly biased. Although such a bias can be somewhat quantified by the approximation guarantee of these methods (see Remark~\ref{rem:naiveCI} for details), we show below that there is a more direct, elegant, and tighter approach.

\para{Derivation of the Interval}
Again, suppose we are given $\Qcal$ such that $Q^\pi \in \Qcal$.\footnote{This condition can be relaxed to $Q^\pi \in \conv(\Qcal)$ due to the affinity of $L(w, \cdot)$.} Then from Lemma~\ref{lem:pel_q},
\begin{align}
J(\pi) = &~ Q^\pi(s_0, \pi) + \EE_w[r + \gamma Q^\pi(s', \pi) - Q^\pi(s,a)] \nonumber \\
\le &~ \supq \Big\{ q(s_0, \pi) + \EE_w[r + \gamma q(s', \pi) - q(s,a)]\Big\}.
\end{align}
For convenience, from now on we will use the shorthand 
\begin{align}
\Lwq := q(s_0, \pi) + \EE_w[r+\gamma q(s',\pi) - q(s,a)],
\end{align}
and the upper bound is then $\sup_{q\in\Qcal} \Lwq$. 
The lower bound is similar:
\begin{align}
J(\pi) \ge \infq \Big\{ q(s_0, \pi) + \EE_w[r + \gamma q(s', \pi) - q(s,a)]\Big\} = \infq \Lwq.
\end{align}
To recap, an arbitrary $w$ will give us a valid interval $[\infq \Lwq,~ \supq \Lwq]$, and we may search over a class $\Wcal$ to find a tighter interval\footnote{We cannot search over the unrestricted (tabular) class when the state space is large, due to overfitting. In contrast, the value bounds derived in this paper only optimize over restricted $\Qcal$ and $\Wcal$ classes, allowing the bound computed on a finite sample to generalize when $\Qcal$ and $\Wcal$ have bounded statistical complexities; see Appendix~\ref{app:staterr} for related discussions.} by taking the \emph{lowest upper bound} and the \emph{highest lower bound}: 
\begin{align}
J(\pi) \le \infw \supq \Lwq =: \ubw, \qquad
J(\pi) \ge \supw \infq \Lwq =: \lbw.
\end{align}
It is worth keeping in mind that the above derivation assumes realizable $\conv(\Qcal)$. Without such an assumption, there is no guarantee that $J(\pi) \le \ubw$, $J(\pi)\ge \lbw$, or even $\lbw \le \ubw$. Below we establish the conditions under which the interval is valid (i.e., $\lbw \le J(\pi) \le \ubw$) and tight (i.e., $\ubw - \lbw$ is small). 

\para{Properties of the Interval}
Intuitively, if $\Qcal$ is richer, it is more likely to be realizable, which improves the interval's validity. 
If we further make $\Wcal$ richer, the interval becomes tighter, as we are searching over a richer space to suppress the upper bound and raise the lower bound. We formalize these intuitions with the theoretical results below. 
Notably, our main results do not require any explicit realizability assumptions and hence are  automatically agnostic. All proofs of this section can be found in App.~\ref{app:proof}.

\begin{theorem}[Validity] \label{thm:ciw_valid}
	Define $\Lmwl(w,q) := q(s_0, \pi) + \EE_w[\gamma q(s', \pi) - q(s,a)]$. We have \vspace*{0.4em}
\begin{align*}
\ubw - J(\pi)  \ge \infw \supq \Lmwl(w, q-Q^\pi), \qquad
J(\pi) - \lbw \ge \infw \supq \Lmwl(w, Q^\pi-q). 
\end{align*}
	As a corollary, when $Q^\pi \in \conv(\Qcal)$, the interval is valid, i.e., $\ubw \ge J(\pi) \ge \lbw$.
\end{theorem}

$\Lmwl(w, q-Q^\pi)$ and $\Lmwl(w, Q^\pi -q)$ can be viewed as a measure of difference between $Q^\pi$ and $q$, as they are essentially linear measurements of $q-Q^\pi$ (note that $\Lmwl(w, \cdot)$ is linear) with the measurement vector determined by $w$. 
Therefore,  if $\conv(\Qcal)$ contains close approximation of $Q^\pi$, then $\ubw$ will not be too much lower than $J(\pi)$ and $\lbw$ will not be too much higher than $\lbw$, and the degree of realizability of $\conv(\Qcal)$ determines to what extent $[\lbw, \ubw]$ is approximately valid. 

Even if valid, the interval may be useless if $\ubw \gg \lbw$. Below we show that this can be prevented by having a well-specified $\Wcal$.

\begin{theorem}[Tightness]
	\label{thm:ciw_tight}
	$
	\ubw - \lbw \le 2 \infw \supq |\Lmwl(w, q)|.
	$ 
	As a corollary, when $\wpi \in \Wcal$, we have $\ubw\le \lbw$.
\end{theorem}

\begin{remark}[Interpretation of Theorem~\ref{thm:ciw_tight}]
$\ubw \le \lbw$ does not imply $\ubw = \lbw$, as 
$\ubw$ and $\lbw$ may lose their upper/lower bound semantics without realizable $\conv(\Qcal)$. An interval that is both tight and valid can only be implied from Theorems~\ref{thm:ciw_valid} and \ref{thm:ciw_tight} together.
\end{remark}

\begin{remark}[Point estimate] \label{rem:point_estimate}
	If a point estimate is desired, we may output $\tfrac{1}{2}(\ubw + \lbw)$, and under $Q^\pi \in \conv(\Qcal)$ we can assert that its error is bounded by $\infw \supq |\Lmwl(w,q)|$. This coincides with the guarantee of MWL under the same assumption \citep[Theorem 2]{uehara2019minimax}. Furthermore, if we simply output $\ubw$ or $\lbw$ (the latter being the policy-evaluation component of Fenchel AlgaeDICE \citep{nachum2019algaedice})\footnote{See Appendix~\ref{app:algaedice} for how to translate between the two papers' notations.}, the approximation guarantee will be twice as large since they only incur one-sided errors. 
\end{remark}

\begin{remark}[Na\"ive interval]\label{rem:naiveCI}
	As we alluded to earlier, the approximation guarantee of \citet[Theorem 2]{uehara2019minimax} can be used to derive an interval that is also valid under realizable $\conv(\Qcal)$; see Table~\ref{tab:compare} for details. 
	In App.~\ref{app:naiveCI}, we show that such an interval is never tighter than ours under realizable $\conv(\Qcal)$. 
\end{remark}

\begin{remark}[Regularization] In App.~\ref{app:reg} we derive the regularized version of the interval via Fenchel transformations, similar to \citep{nachum2019algaedice, nachum2020reinforcement}. 
\end{remark}

\subsection{Value Interval for Well-specified $\Wcal$ and Misspecified $\Qcal$} \label{sec:mql-ci}
Similar to Sec.~\ref{sec:mwl-ci}, we now derive the interval for the case of realizable $\conv(\Wcal)$. We also base our entire derivation on the following simple lemma, which can be used to derive the \lqfw method.

\begin{lemma}[Evaluation Error Lemma for Value Functions]\label{lem:pel_w}
	For any $q:\Scal\times\Acal \to \RR$, 
	\begin{align} \label{eq:pel_w}
	J(\pi) -  q(s_0, \pi) = \EE_{d^\pi}[r + \gamma q(s', \pi) - q(s, a)].
	\end{align}
\end{lemma}
\begin{proof}
$J(\pi) = \EE_{d^\pi}[r]$, and $\EE_{d^\pi}[q(s,a) - \gamma q(s', \pi)] = q(s_0, \pi)$ due to Bellman equation for $d^\pi$. 
\end{proof}
\para{\Lqfw in a Nutshell} MQL \citep{uehara2019minimax} seeks to find $q$ such that $q(s_0, \pi) \approx J(\pi)$. Using Lemma~\ref{lem:pel_w}, this can be achieved by finding $q$ that sets the RHS of Eq.\eqref{eq:pel_w} to $0$, and we can use a class $\Wcal$ that realizes $\wpi$ to overcome the difficulty of unknown $d^\pi$, similar to how we handle the unknown $Q^\pi$ in Sec.~\ref{sec:mwl-ci}. While the method gives an accurate estimation when both $\Qcal$ and $\Wcal$ are well-specified, below we show how to derive an interval to quantify the bias due to misspecified $\Qcal$.

\para{Derivation of the Interval} 
If we are given $\Wcal \subset (\Scal\times\Acal\to\RR)$ such that $\wpi \in \Wcal$, then\footnote{As before, this condition can be relaxed to $\wpi \in \conv(\Wcal)$.}
\begin{align}
& J(\pi) \le q(s_0, \pi) + \sup_{w\in\Wcal} \EE_w[r + \gamma q(s', \pi) - q(s,a)] = \supw \Lwq. \\
& J(\pi) \ge q(s_0, \pi) + \inf_{w\in\Wcal} \EE_w[r + \gamma q(s', \pi) - q(s,a)] = \infw \Lwq.
\end{align}
Again, an arbitrary $q$ yields a valid interval, and we may search over a class $\Qcal$ to tighten it: 
\begin{align}
J(\pi) \le \inf_{q\in\Qcal} \sup_{w\in\Wcal} \Lwq =: \ubq, \qquad
J(\pi) \ge \sup_{q\in\Qcal} \inf_{w\in\Wcal} \Lwq =: \lbq.
\end{align}

\para{Properties of the Interval} We characterize the interval's validity and tightness similar to Sec.~\ref{sec:mwl-ci}.

\begin{theorem}[Validity] \label{thm:ciq_valid}
	Define $\Lmql(w,q) := \EE_w[r + \gamma q(s', \pi) - q(s,a)]$. 
	\begin{align*}
\ubq - J(\pi)  \ge \infq \supw \, \Lmql(w-\wpi, q),  \qquad 
J(\pi) - \lbq \ge \infq \supw \, \Lmql(\wpi-w, q). 
	\end{align*}
	As a corollary, when $\wpi \in \conv(\Wcal)$, the interval is valid, i.e., $\ubw \ge J(\pi) \ge \lbw$.
\end{theorem}
Again, the validity of the interval is controlled by the realizability of $\conv(\Wcal)$, defined as the best approximation of $\wpi \in \Wcal$ where the difference between $\wpi$ and any $w$ is measured by $\Lmql(w-\wpi, q)$ and $\Lmql(\wpi-w, q)$ (which are linear measurements of $w-\wpi$).

\begin{theorem}[Tightness]
	\label{thm:ciq_tight}
	$
	\ubq - \lbq \le 2 \infq \supw |\Lmql(w, q)|.
	$ As a corollary, when $\qpi \in \Qcal$, we have $\ubq\le \lbq$.
\end{theorem}

\subsection{Unification} \label{sec:unify}
So far we have obtained two intervals:
\begin{align*}
\textrm{$\conv(\Qcal)$ realizable:} \begin{cases}
\ubw = \inf_w \sup_q \Lwq, \\
\lbw = \sup_w \inf_q \Lwq,
\end{cases} & \quad
\textrm{$\conv(\Wcal)$ realizable:} \begin{cases}
\ubq = \inf_q \sup_w \Lwq, \\
\lbq = \sup_q \inf_w \Lwq,
\end{cases}
\end{align*}
where $\Lwq = q(s_0, \pi) + \EE_w[r+\gamma q(s',\pi) - q(s,a)]$. Taking a closer look, $\ubw$ and $\lbq$ are almost the same and only differ in the order of optimizing $q$ and $w$, and so are $\ubq$ and $\lbw$. It turns out that if $\Wcal$ and $\Qcal$ are convex, these two intervals are precisely the \emph{reverse} of each other.
\begin{theorem} \label{thm:convex}
	If $\Wcal$ and $\Qcal$ are compact and convex sets, we have
	$
	\ubw = \lbq, ~~ \ubq = \lbw.
	$
\end{theorem}
This result implies that we do not need to separately consider these two intervals. Neither do we need to know which one of $\Qcal$ and $\Wcal$ is well specified. We just need to do the obvious thing, which is to compute $\ubw(=\lbq)$ and $\ubq(=\lbw)$, and let the smaller number be the lower bound and the greater one be the upper bound. This way we get a single interval that is valid when either $\Qcal$ or $\Wcal$ is realizable (Theorems~\ref{thm:ciw_valid} and \ref{thm:ciq_valid}), and tight when both are. Furthermore, by looking at which value is lower, we can tell which class is misspecified (assuming one of them is well-specified), and this piece of information may provide guidance to the design of function approximation. 

\para{The Non-convex Case} When $\Qcal$ and $\Wcal$ are non-convex, the two intervals are still related in an interesting albeit more subtle manner: the two ``reversed intervals'' are generally tighter than the original intervals, but they are only valid under stronger realizability conditions. See proofs and further discussions in  App.~\ref{app:nonconvex}.
\begin{theorem} \label{thm:nonconvex}
	When $\qpi \in \Qcal$, $J(\pi) \in [\ubq, \lbq] \subseteq [\lbw, \ubw]$. When $\wpi \in \Wcal$, \\$J(\pi) \in [\ubw, \lbw]\subseteq [\lbq, \ubq]$.
\end{theorem}

\subsection{Empirical Verification of Theoretical Predictions} \label{sec:ope_exp}
We provide  preliminary empirical results to support the theoretical predictions, that \\
(1) which bound is the upper bound depends on the expressivity of function classes (Sec.~\ref{sec:unify}), and \\
(2) our interval is tighter than the na\"ive intervals based on previous methods (App.\ref{app:naiveCI}). \\
We conduct the experiments in CartPole, with the target policy being softmax over a pre-trained $Q$-function with temperature $\tau$ (behavior policy is $\tau=1.0$). We use neural nets for $\Qcal$ and $\Wcal$, and optimize the losses using stochastic gradient descent ascent (SGDA);\footnote{Neural nets induce non-convex classes, violating Theorem~\ref{thm:convex}'s assumptions. However, such non-convexity may be mitigated by having an ensemble of neural nets, or in the recently popular infinite-width regime \citep{jacot2018neural}. Also, SGDA is symmetric w.r.t.~$w$ and $q$, and our implementation can be viewed as heuristic approximations of $\ubq$ \& $\lbw$ and $\ubw$ \& $\lbq$, resp., so we treat $\ubq \approx \lbw$ and $\ubw \approx \lbq$ in Sec.~\ref{sec:ope_exp}.}
see App.~\ref{app:exp} for more details. 
Fig.~\ref{fig:reverse} demonstrates the interval reversal phenomenon, where we compute $\lbq (\approx \ubw)$ and $\ubq (\approx \lbw)$ for  Q-networks of different sizes while fixing everything else. As predicted by theory, $\ubq > \lbq$ when $\Qcal$ is small (hence poorly specified), and the interval is flipped as $\Qcal$ increases.



\begin{figure}[ht!]
	\centering
\begin{minipage}[c]{0.38\textwidth}
	\includegraphics[width=\textwidth]{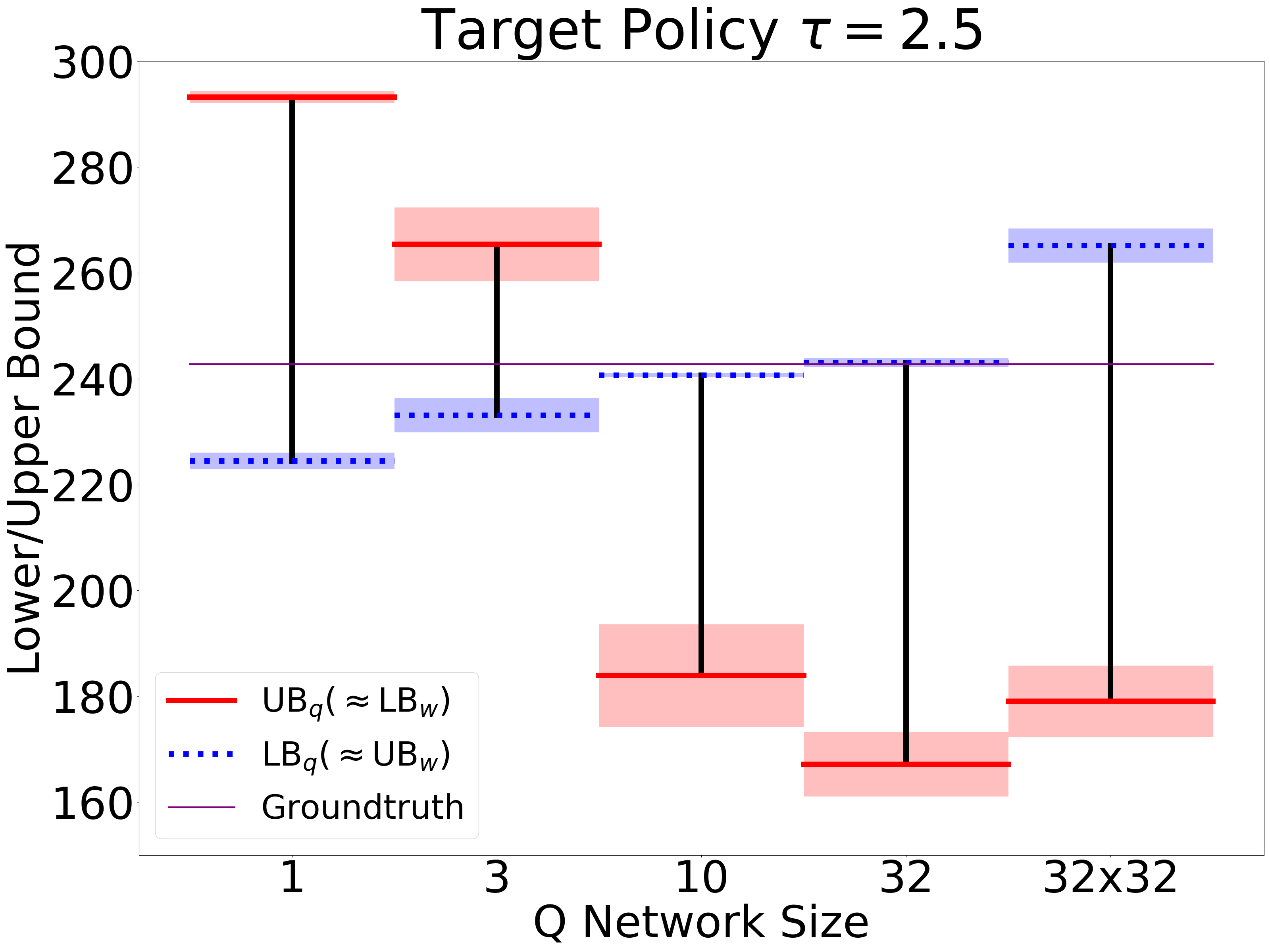}
\end{minipage}
\begin{minipage}[c]{0.24\textwidth}
\caption{ \textbf{(Left)} The interval reversal phenomenon; see Sec.~\ref{sec:ope_exp}. \label{fig:reverse}	} \vspace*{-1em}
\caption{\textbf{(Right)} Policy optimization  under non-exploratory data; see Sec.~\ref{sec:opt_exp} for details. \label{fig:opt}}
\end{minipage} \hfill
\begin{minipage}[c]{0.36\textwidth}
	\includegraphics[width=\textwidth]{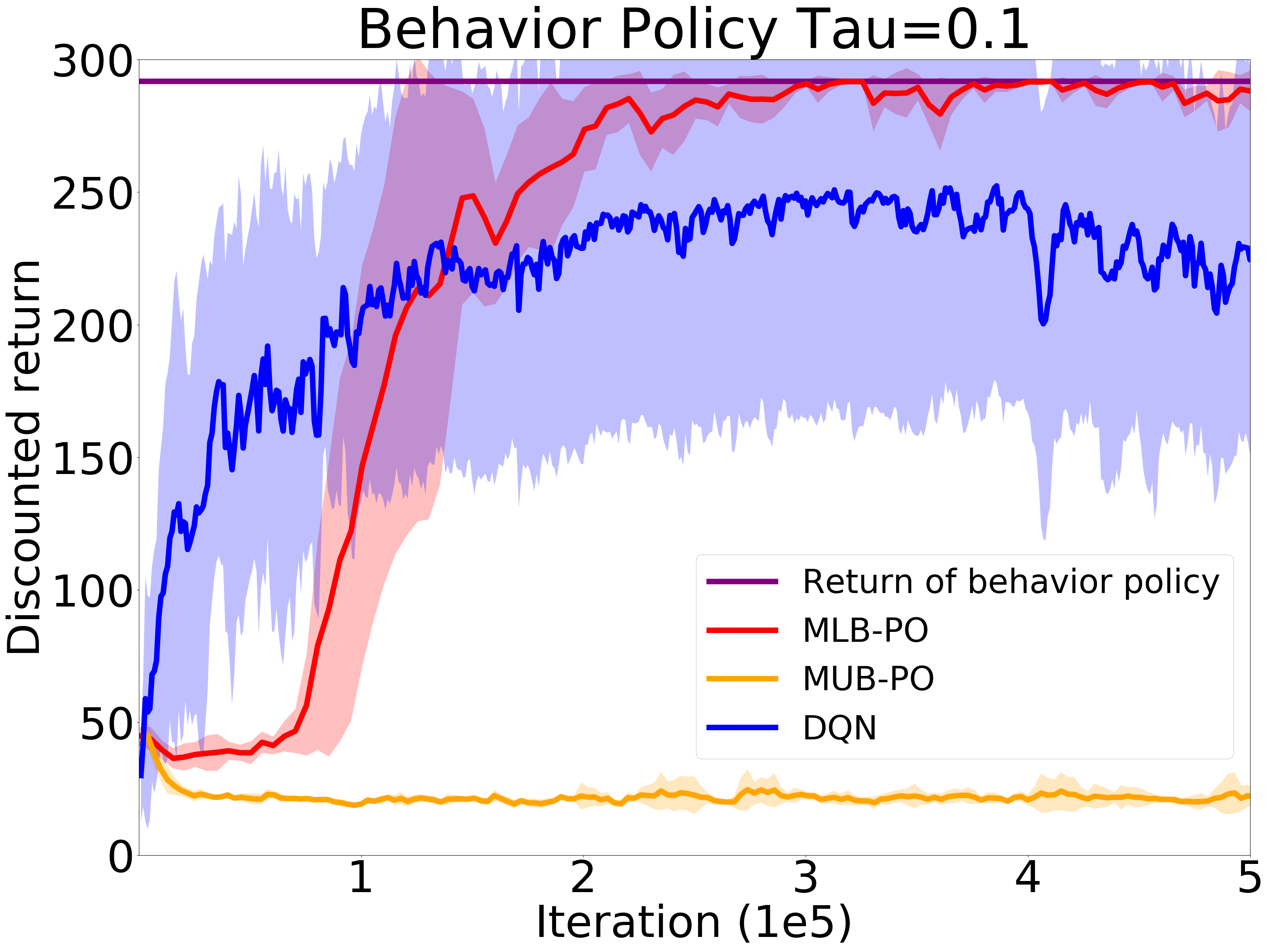}
\end{minipage}
\end{figure}
\begin{figure}[ht!]
\begin{minipage}[c]{0.24\textwidth}
\caption{Comparison of interval length and validity ratio between our interval and that of MQL as a function of sample size; see App.~\ref{app:bootstrap} for more details. \label{fig:len} All error bars show twice standard errors.}
\end{minipage}\hfill
\begin{minipage}[c]{0.74\textwidth}
	\includegraphics[width=0.48\textwidth]{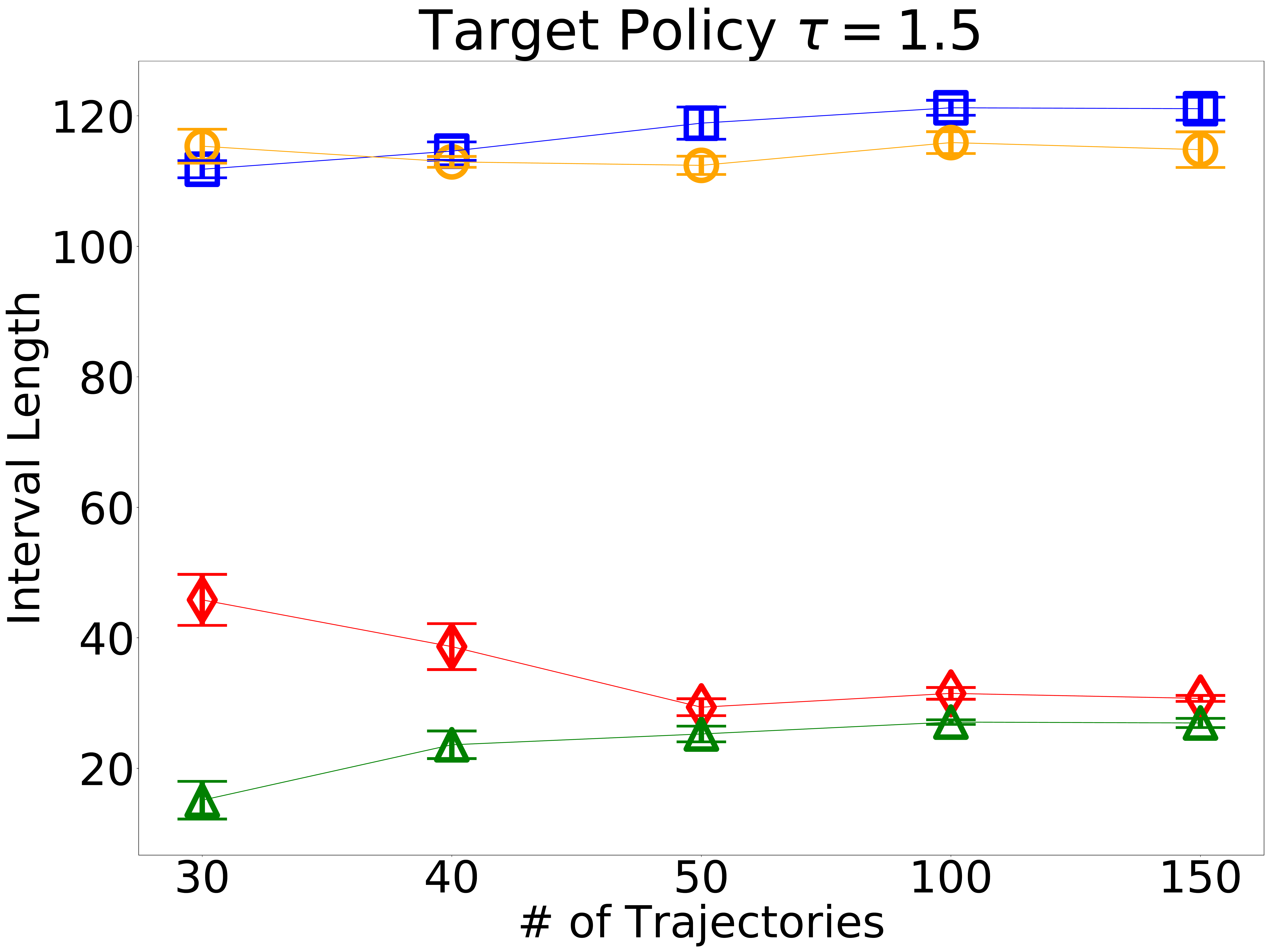}
	\includegraphics[width=0.48\textwidth]{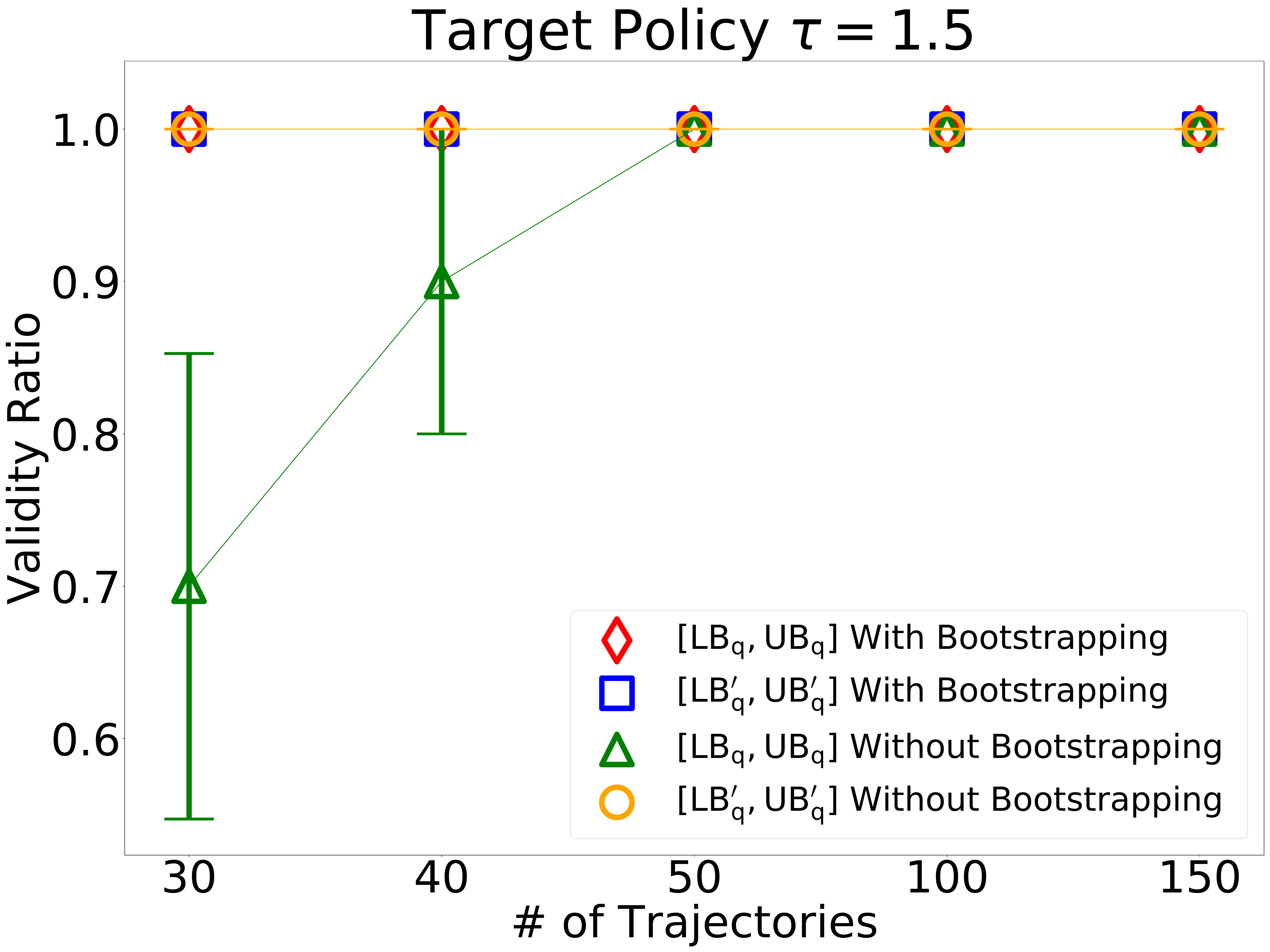}
\end{minipage} 
\end{figure}

Fig.~\ref{fig:len} (left) compares the interval lengths between our interval and that induced by MQL under different sample sizes,\footnote{We compare to the interval induced by MQL ($[\lbq', \ubq']$ in Table~\ref{tab:compare}), as \citet{uehara2019minimax} reported that MQL is more stable and performs better than MWL in discounted problems.}
and we see the former is significantly tighter than the latter. However, Fig.~\ref{fig:len} (right) reveals that when the sample size is small, our new interval is too aggressive and fails to contain $J(\pi)$ with a significant chance. This is also expected as our theory does not handle statistical errors, and validity guarantee can be violated due to data randomness especially in the small sample regime. While  a theoretical investigation of the statistical errors (beyond the use of generalization error bounds in App.~\ref{app:staterr} which can be loose) is left to future work, we empirically test a popular heuristic of bootstrapping intervals (App.\ref{app:bootstrap}), and show that after bootstrapping our interval becomes $100\%$ valid, while its length has only increased moderately especially in the large sample regime.

\section{On Policy Optimization with Insufficient Data  Coverage} \label{sec:opt}
OPE is not only useful as an evaluation method, but can also serve as a crucial component for off-policy policy optimization: if we can estimate the return $J(\pi)$ for a class of policies $\Pi$, then in principle we may use $\argmax_{\pi \in \Pi} J(\pi)$ to find the best policy in class. However, our interval gives two estimations of $J(\pi)$, an upper bound and a lower bound. Which one should we optimize? 
Traditional wisdom in robust MDP literature \citep[e.g.,][]{nilim2005robust} suggests \emph{pessimism}, i.e., optimizing the lower bound for the best worst-case guarantees. \emph{But which side of the interval is the lower bound?} 

As Sec.\ref{sec:unify} indicates, every expression could be the upper bound or the lower bound, depending on the realizability of $\Qcal$ and $\Wcal$. So the question becomes: which one of the function classes is more likely to be misspecified? 

\para{RL with Insufficient Data Coverage} 
Even if $\Wcal$ and $\Qcal$ are chosen with maximum care, the realizability of $\Wcal$ faces one additional and crucial challenge compared to $\Qcal$: 
if the data distribution $\mu$ itself does not  cover the support of $d^\pi$, then $\wpi$ may not even exist and hence $\Wcal$ will be poorly specified whatsoever. Since the historical dataset in real applications often comes with no exploratory guarantees, this scenario is highly likely and remains as a crucial yet understudied challenge \citep{fujimoto2019off, liu2019off}.  
In this section we analyze the algorithms that optimize the upper and the lower bounds when $\Wcal$ is poorly specified. For convenience we will call them\footnote{We make the dependence of $\ubw$ and $\lbw$ on $\pi$ explicit since we consider multiple policies in this section.}
\begin{align}
\textrm{MUB-PO:} ~~\argmax_{\pi \in \Pi}~ \ubwpi, \qquad \textrm{MLB-PO:}~~ \argmax_{\pi \in \Pi}~ \lbwpi,
\end{align}
where MUB-PO/MLB-PO stand for minimax upper (lower) bound policy optimization. MLB-PO is essentially Fenchel AlgaeDICE \citep{nachum2019algaedice}, so our results provide further justification for this method. On the other hand, while MUB-PO may not be appropriate for exploitation, it exercises \emph{optimism} and may be better applied to induce exploration, that is, collecting new data with the policy computed by MUB-PO may improve the data coverage.

\subsection{Case Study: Exploration and Exploitation in a Tabular Scenario} \label{sec:tabular}
To substantiate the claim that MUB-PO/MLB-PO induce effective exploration/exploitation, we first consider a scenario where the solution concepts for exploration and exploitation are clearly understood, and show that MUB-PO/MLB-PO reduce to familiar algorithms. 

\para{Setup} Consider an MDP with finite and discrete state and action spaces. Let $\Sk \subset \Scal$ be a subset of the state space, and $s_0 \in \Sk$. Suppose the dataset contains $n$ transition samples from each $s \in \Sk, a\in\Acal$ with $n \to \infty$ (i.e., the transitions and rewards in $\Sk$ are fully known), and $0$ samples from $s \notin \Sk$.

\para{Known Solution Concepts} In such a simplified setting, effective exploration can be achieved by (the episodic version of) the  well-known Rmax algorithm \citep{brafman2003r,kakade2003sample}, which computes the optimal policy of the Rmax-MDP: this MDP has the same transition dynamics and reward function as the true MDP $M$ on states with sufficient data ($\Sk$), and the ``unknown'' states are assumed to have self-loops with $\Rmax$ rewards, making them appealing to visit and thus encouraging exploration. 
Similarly, when the goal is to output a policy with the best worst-case guarantee (exploitation), one simply changes $\Rmax$  to the minimally possible reward (``$\Rmin$'', which is $0$ for us).  
Below we show that MUB-PO and MLB-PO in this setting precisely correspond to the Rmax and the Rmin algorithms, respectively. 
\begin{proposition}\label{prop:Rmin}
	Consider the MDP and the dataset described above. Let $\Wcal = [0, \frac{|\Sk\times\Acal|}{1-\gamma}]^{\Scal\times\Acal}$ 
	and $\Qcal = [0, \Rmax/(1-\gamma)]^{\Scal\times\Acal}$. In this case, MUB-PO reduces to Rmax, and MLB-PO reduces to Rmin. 
\end{proposition}
Despite the intuitiveness of the statement, the proof is quite involved and requires repeated applications of the results established in Sec.~\ref{sec:ope}; we refer interested readers to App.~\ref{app:rmax} for proof details.

\subsection{Guarantees in the Function Approximation Setting}  \label{sec:opt_fa}
We give more general guarantee in the function approximation setting. For simplicity we do not consider e.g., approximation/estimation errors, and incorporating them is routine \citep{munos2008finite,farahmand2010error,chen2019information}.


\para{Exploitation with Well-specified $\Qcal$} We start with the guarantee of MLB-PO. 
\begin{proposition}[MLB-PO] \label{prop:mlbpo}
	Let $\Pi$ be a policy class, and assume $Q^\pi \in \conv(\Qcal) ~ \forall \pi \in \Pi$. Let $\hat \pi = \argmax_{\pi\in\Pi} \lbwpi$. Then, for any $\pi\in\Pi$,
	$
	J(\hat \pi) \ge \lbwpi. 
	$ 
	As a corollary, for any $\pi$ s.t.~$\wpi \in \Wcal$,  $J(\hat \pi) \ge J(\pi)$, that is, we compete with any policy whose importance weight is realized by $\Wcal$.
\end{proposition}

\para{Exploration with (Less) Well-specified $\Qcal$} We then provide the exploration guarantee of MUB-PO, 
which is an ``optimal-or-explore'' statement, that either the obtained policy is near-optimal (to be defined below), or it will induce effective exploration. 
Perhaps surprisingly, our results suggest that MUB-PO for exploration might be significantly more robust against misspecified $\Qcal$ than MLB-PO for exploitation. 

The key idea behind the agnostic result is the following: instead of competing with $\max_{\pi\in\Pi} J(\pi)$ as the optimal value under the assumption that $Q^\pi \in \conv(\Qcal), \forall \pi\in\Pi$ (the same assumption as MLB-PO), we aim at a less ambitious notion of optimality under a substantially relaxed assumption; without any explicit assumption on $\Qcal$, we directly compete with
$ 
\max_{\pi\in\Pi: Q^\pi \in \conv(\Qcal)} J(\pi).
$ 
In words, we compete with any policy whose Q-function is realized by $\conv(\Qcal)$. When $Q^{\pi} \in \conv(\Qcal)$ for $\pi = \argmax_{\pi\in\Pi} J(\pi)$, we compete with the usual notion of optimal value. 
However, even if some (or most) policies' Q-functions elude $\conv(\Qcal)$, we can still compete with whichever policy whose Q-function is captured by $\conv(\Qcal)$. A similar notion of optimality has been used by \citet{jiang2017contextual}, and indeed their algorithm is closely related to MUB-PO, which we discuss in App.~\ref{app:olive}. 

We state a short version of MUB-PO's guarantee, with the full version deferred to App.~\ref{app:mubpo}. 
\begin{proposition}[MUB-PO, short ver.] \label{prop:mubpo}
	Let $\Pi$ be a policy class. Let $\hat \pi = \argmax_{\pi\in\Pi} \ubwpi$. Assuming $\|q\|_\infty \le \Rmax/(1-\gamma)~ \forall q\in\Qcal$, we have for any $w\in \Wcal$,
	$$
	\|w\cdot \mu - d^{\hat \pi}\|_1 \ge \frac{(1-\gamma) \left(\max_{\pi \in \Pi: Q^\pi\in\conv(\Qcal)} J(\pi) - J(\hat \pi)\right)}{2\Rmax}.
	$$
\end{proposition}
Recall that $w\in\Wcal$ is supposed to model the importance weight that coverts data $\mu$ to the occupancy of some policy, e.g., $\wpi \,\cdot\, \mu = d^\pi$. The proposition states that either $\hat \pi$ is near-optimal, or it will induce an occupancy that cannot be accurately modeled by \emph{any} importance weights in $\Wcal$ when applied on the current data distribution $\mu$. 
Hence, if $\Wcal$ is very rich and models all distributions covered by $\mu$, then $\hat{\pi}$ must visit new state-actions or it must be near-optimal. 


\subsection{Preliminary Empirical Results} \label{sec:opt_exp}
While we would like to test MLB-PO and MUB-PO in experiments, the joint optimization of $\pi, w, q$ is very challenging and remains an open problem \citep{nachum2019algaedice}. Similar to \citet{nachum2019algaedice}, we try a heuristic variant of MLB-PO and MUB-PO that works in near-deterministic environments; see App.~\ref{app:opt_exp} for details. As Fig.~\ref{fig:learning_curve} shows, when the behavior policy is non-exploratory ($\tau=0.1$ yields a nearly deterministic policy), MLB-PO can reliably achieve good performance despite the lack of explicit regularization towards the behavior policy, whereas DQN is more unstable and suffers higher variance. MUB-PO, on the other hand, fails to achieve a high value---which is also expected from theory---but is able to induce exploration in certain cases. We defer more detailed results and discussions to App.~\ref{app:opt_exp} due to space limit. 

\section{Conclusions and Open Problems}
We derive a minimax value interval for  off-policy evaluation. The interval is valid as long as either the importance-weight or the value-function class is well specified, and its length quantifies the misspecification error of the other class. Our highly simplified derivations only take a few steps from the basic Bellman equations, which condense and unify the derivations in existing works. 
When applied to off-policy policy optimization in face of insufficient data coverage, which is an important scenario of practical concerns, optimizing our lower and upper bounds over a policy class can induce effective exploitation and exploration, respectively. 

We conclude the paper with open problems for future work: 
\begin{itemize}[leftmargin=*]
\item We handled sampling errors via bootstrapping in the experiments. Are there statistically more effective and computationally more efficient solutions? 
\item MLB-PO and MUB-PO exhibit promising statistical properties but require difficult optimization. Can we develop principled and practically effective  strategies for optimizing these objectives?
\item The double robustness of our interval protects against the misspecification of either $\Qcal$ or $\Wcal$, but still requires one of them to be realizable to guarantee validity. Is it possible at all to develop an interval that is \emph{always} valid, and whose length is allowed to depend on the misspecification errors of both $\Qcal$ and $\Wcal$? If impossible, can we establish information-theoretic hardness, and does reformulating the problem in a practically relevant manner help circumvent the difficulty? 
\end{itemize}  
Answering these questions will be important steps towards reliable and practically useful off-policy evaluation. 

\section*{Broader Impact}
This work is largely of theoretical nature, trying to unify existing methods and pointing out their connections, with minimal proof-of-concept simulation experiments. Therefore, we do not foresee direct broader impact. That said, an important motivation for this work is to equip RL with off-line evaluation methods that rely on as few assumptions as possible, and in the long term this should contribute to a more  trustworthy framework for applying RL to real-world tasks, where reliable evaluation is indispensable. We warn, however, that even though we aim at a less ambitious goal of producing a valid interval (whose length may not go to $0$ as sample size increases), we still require unverifiable assumptions (realizability of $\conv(\Qcal)$ or $\conv(\Wcal)$). Therefore, the value intervals produced by this and subsequent papers should be interpreted and treated with care and not taken as-is in application scenarios. There are also several important aspects of building practically useful confidence intervals that are ignored in this paper (since we are still in the early stage of theoretical investigations), such as the handling of statistical errors and possible confoundedness in the data, which need to be addressed by future works before these methods can be readily deployed in applications. 

\begin{ack}
This project is partially supported by a Microsoft Azure University Grant. The authors thank Jinglin Chen for pointing out several mistakes/typos in an earlier draft of the paper. 
\end{ack}

\bibliography{bib,pfi3}

\begin{thebibliography}{37}
\providecommand{\natexlab}[1]{#1}
\providecommand{\url}[1]{\texttt{#1}}
\expandafter\ifx\csname urlstyle\endcsname\relax
  \providecommand{\doi}[1]{doi: #1}\else
  \providecommand{\doi}{doi: \begingroup \urlstyle{rm}\Url}\fi

\bibitem[Uehara et~al.(2020)Uehara, Huang, and Jiang]{uehara2019minimax}
Masatoshi Uehara, Jiawei Huang, and Nan Jiang.
\newblock {Minimax Weight and Q-Function Learning for Off-Policy Evaluation}.
\newblock In \emph{Proceedings of the 37th International Conference on Machine
  Learning (ICML-20)}, 2020.

\bibitem[Precup et~al.(2000)Precup, Sutton, and Singh]{precup2000eligibility}
Doina Precup, Richard~S Sutton, and Satinder~P Singh.
\newblock {Eligibility Traces for Off-Policy Policy Evaluation}.
\newblock In \emph{Proceedings of the 17th International Conference on Machine
  Learning}, pages 759--766, 2000.

\bibitem[Jiang and Li(2016)]{jiang2016doubly}
Nan Jiang and Lihong Li.
\newblock {Doubly Robust Off-policy Value Evaluation for Reinforcement
  Learning}.
\newblock In \emph{Proceedings of The 33rd International Conference on Machine
  Learning}, volume~48, pages 652--661, 2016.

\bibitem[Li et~al.(2015)Li, Munos, and Szepesv{\'a}ri]{li2015minimax}
Lihong Li, R{\'e}mi Munos, and Csaba Szepesv{\'a}ri.
\newblock Toward minimax off-policy value estimation.
\newblock In \emph{Proceedings of the 18th International Conference on
  Artificial Intelligence and Statistics}, 2015.

\bibitem[Liu et~al.(2018)Liu, Li, Tang, and Zhou]{liu2018breaking}
Qiang Liu, Lihong Li, Ziyang Tang, and Dengyong Zhou.
\newblock Breaking the curse of horizon: Infinite-horizon off-policy
  estimation.
\newblock In \emph{Advances in Neural Information Processing Systems}, pages
  5361--5371, 2018.

\bibitem[Feng et~al.(2020)Feng, Ren, Tang, and Liu]{feng2020accountable}
Yihao Feng, Tongzheng Ren, Ziyang Tang, and Qiang Liu.
\newblock Accountable off-policy evaluation with kernel bellman statistics.
\newblock In \emph{Proceedings of the 37th International Conference on Machine
  Learning (ICML-20)}, 2020.

\bibitem[Dai et~al.(2020)Dai, Nachum, Chow, Li, Szepesvari, and
  Schuurmans]{dai2020coindice}
Bo~Dai, Ofir Nachum, Yinlam Chow, Lihong Li, Csaba Szepesvari, and Dale
  Schuurmans.
\newblock {CoinDICE: Off-Policy Confidence Interval Estimation}.
\newblock In \emph{Advances in neural information processing systems}, 2020.

\bibitem[Xie et~al.(2019)Xie, Ma, and Wang]{XieTengyang2019OOEf}
Tengyang Xie, Yifei Ma, and Yu-Xiang Wang.
\newblock Towards optimal off-policy evaluation for reinforcement learning with
  marginalized importance sampling.
\newblock In \emph{Advances in Neural Information Processing Systems 32}, pages
  9665--9675. 2019.

\bibitem[Nachum et~al.(2019{\natexlab{a}})Nachum, Chow, Dai, and
  Li]{ChowYinlam2019DBEo}
Ofir Nachum, Yinlam Chow, Bo~Dai, and Lihong Li.
\newblock Dualdice: Behavior-agnostic estimation of discounted stationary
  distribution corrections.
\newblock In \emph{Advances in Neural Information Processing Systems 32}.
  2019{\natexlab{a}}.

\bibitem[Liu et~al.(2020)Liu, Bacon, and Brunskill]{liu2019understanding}
Yao Liu, Pierre-Luc Bacon, and Emma Brunskill.
\newblock Understanding the curse of horizon in off-policy evaluation via
  conditional importance sampling.
\newblock In \emph{Proceedings of the 37th International Conference on Machine
  Learning (ICML-20)}, 2020.

\bibitem[Liu et~al.(2019)Liu, Swaminathan, Agarwal, and Brunskill]{liu2019off}
Yao Liu, Adith Swaminathan, Alekh Agarwal, and Emma Brunskill.
\newblock Off-policy policy gradient with state distribution correction.
\newblock In \emph{Proceedings of the 35th Conference on Uncertainty in
  Artificial Intelligence (UAI-19)}, 2019.

\bibitem[Rowland et~al.(2020)Rowland, Harutyunyan, van Hasselt, Borsa, Schaul,
  Munos, and Dabney]{rowland2019conditional}
Mark Rowland, Anna Harutyunyan, Hado van Hasselt, Diana Borsa, Tom Schaul,
  R{\'e}mi Munos, and Will Dabney.
\newblock Conditional importance sampling for off-policy learning.
\newblock 108:\penalty0 45--55, 2020.

\bibitem[Liao et~al.(2019)Liao, Klasnja, and Murphy]{liao2019off}
Peng Liao, Predrag Klasnja, and Susan Murphy.
\newblock Off-policy estimation of long-term average outcomes with applications
  to mobile health.
\newblock \emph{arXiv preprint arXiv:1912.13088}, 2019.

\bibitem[Zhang et~al.(2019)Zhang, Dai, Li, and Schuurmans]{zhang2019gendice}
Ruiyi Zhang, Bo~Dai, Lihong Li, and Dale Schuurmans.
\newblock Gendice: Generalized offline estimation of stationary values.
\newblock In \emph{International Conference on Learning Representations}, 2019.

\bibitem[Zhang et~al.(2020)Zhang, Liu, and Whiteson]{zhang2020gradientdice}
Shangtong Zhang, Bo~Liu, and Shimon Whiteson.
\newblock {GradientDICE: Rethinking Generalized Offline Estimation of
  Stationary Values}.
\newblock In \emph{Proceedings of the 37th International Conference on Machine
  Learning (ICML-20)}, 2020.

\bibitem[Feng et~al.(2019)Feng, Li, and Liu]{feng2019kernel}
Yihao Feng, Lihong Li, and Qiang Liu.
\newblock A kernel loss for solving the bellman equation.
\newblock In \emph{Advances in Neural Information Processing Systems}, pages
  15430--15441, 2019.

\bibitem[Dud{\'\i}k et~al.(2011)Dud{\'\i}k, Langford, and Li]{dudik2011doubly}
Miroslav Dud{\'\i}k, John Langford, and Lihong Li.
\newblock Doubly {R}obust {P}olicy {E}valuation and {L}earning.
\newblock In \emph{Proceedings of the 28th International Conference on Machine
  Learning}, pages 1097--1104, 2011.

\bibitem[Kallus and Uehara(2019)]{KallusNathan2019EBtC}
Nathan Kallus and Masatoshi Uehara.
\newblock Efficiently breaking the curse of horizon: Double reinforcement
  learning in infinite-horizon processes.
\newblock \emph{arXiv preprint arXiv:1909.05850}, 2019.

\bibitem[Tang et~al.(2020)Tang, Feng, Li, Zhou, and Liu]{tang2019harnessing}
Ziyang Tang, Yihao Feng, Lihong Li, Dengyong Zhou, and Qiang Liu.
\newblock Harnessing infinite-horizon off-policy evaluation: Double robustness
  via duality.
\newblock \emph{ICLR 2020(To appear)}, 2020.

\bibitem[Nachum et~al.(2019{\natexlab{b}})Nachum, Dai, Kostrikov, Chow, Li, and
  Schuurmans]{nachum2019algaedice}
Ofir Nachum, Bo~Dai, Ilya Kostrikov, Yinlam Chow, Lihong Li, and Dale
  Schuurmans.
\newblock Algaedice: Policy gradient from arbitrary experience.
\newblock \emph{arXiv preprint arXiv:1912.02074}, 2019{\natexlab{b}}.

\bibitem[Nachum and Dai(2020)]{nachum2020reinforcement}
Ofir Nachum and Bo~Dai.
\newblock Reinforcement learning via fenchel-rockafellar duality.
\newblock \emph{arXiv preprint arXiv:2001.01866}, 2020.

\bibitem[Jacot et~al.(2018)Jacot, Gabriel, and Hongler]{jacot2018neural}
Arthur Jacot, Franck Gabriel, and Cl{\'e}ment Hongler.
\newblock Neural tangent kernel: Convergence and generalization in neural
  networks.
\newblock In \emph{Advances in neural information processing systems}, pages
  8571--8580, 2018.

\bibitem[Nilim and El~Ghaoui(2005)]{nilim2005robust}
Arnab Nilim and Laurent El~Ghaoui.
\newblock Robust control of markov decision processes with uncertain transition
  matrices.
\newblock \emph{Operations Research}, 53\penalty0 (5):\penalty0 780--798, 2005.

\bibitem[Fujimoto et~al.(2019)Fujimoto, Meger, and Precup]{fujimoto2019off}
Scott Fujimoto, David Meger, and Doina Precup.
\newblock Off-policy deep reinforcement learning without exploration.
\newblock In \emph{International Conference on Machine Learning}, pages
  2052--2062, 2019.

\bibitem[Brafman and Tennenholtz(2003)]{brafman2003r}
Ronen~I Brafman and Moshe Tennenholtz.
\newblock R-max-a general polynomial time algorithm for near-optimal
  reinforcement learning.
\newblock \emph{The Journal of Machine Learning Research}, 3:\penalty0
  213--231, 2003.

\bibitem[Kakade(2003)]{kakade2003sample}
Sham~Machandranath Kakade.
\newblock \emph{On the sample complexity of reinforcement learning}.
\newblock PhD thesis, University of College London, 2003.

\bibitem[Munos and Szepesv{\'a}ri(2008)]{munos2008finite}
R{\'e}mi Munos and Csaba Szepesv{\'a}ri.
\newblock Finite-time bounds for fitted value iteration.
\newblock \emph{Journal of Machine Learning Research}, 9\penalty0
  (May):\penalty0 815--857, 2008.

\bibitem[Farahmand et~al.(2010)Farahmand, Szepesv{\'a}ri, and
  Munos]{farahmand2010error}
Amir-massoud Farahmand, Csaba Szepesv{\'a}ri, and R{\'e}mi Munos.
\newblock Error {P}ropagation for {A}pproximate {P}olicy and {V}alue
  {I}teration.
\newblock In \emph{Advances in Neural Information Processing Systems}, pages
  568--576, 2010.

\bibitem[Chen and Jiang(2019)]{chen2019information}
Jinglin Chen and Nan Jiang.
\newblock Information-theoretic considerations in batch reinforcement learning.
\newblock In \emph{Proceedings of the 36th International Conference on Machine
  Learning}, pages 1042--1051, 2019.

\bibitem[Jiang et~al.(2017)Jiang, Krishnamurthy, Agarwal, Langford, and
  Schapire]{jiang2017contextual}
Nan Jiang, Akshay Krishnamurthy, Alekh Agarwal, John Langford, and Robert~E.
  Schapire.
\newblock Contextual decision processes with low {B}ellman rank are
  {PAC}-learnable.
\newblock In \emph{International Conference on Machine Learning}, 2017.

\bibitem[Neumann(1928)]{neumann1928theorie}
J~v Neumann.
\newblock Zur theorie der gesellschaftsspiele.
\newblock \emph{Mathematische annalen}, 100\penalty0 (1):\penalty0 295--320,
  1928.

\bibitem[Sion et~al.(1958)]{sion1958general}
Maurice Sion et~al.
\newblock On general minimax theorems.
\newblock \emph{Pacific Journal of mathematics}, 8\penalty0 (1):\penalty0
  171--176, 1958.

\bibitem[Voloshin et~al.(2019)Voloshin, Le, Jiang, and
  Yue]{voloshin2019empirical}
Cameron Voloshin, Hoang~M Le, Nan Jiang, and Yisong Yue.
\newblock Empirical study of off-policy policy evaluation for reinforcement
  learning.
\newblock \emph{arXiv preprint arXiv:1911.06854}, 2019.

\bibitem[Jiang(2018)]{nan_rmax_notes}
Nan Jiang.
\newblock \emph{{CS 598: Notes on Rmax exploration}}.
\newblock {University of Illinois at Urbana-Champaign}, 2018.
\newblock \url{http://nanjiang.cs.illinois.edu/files/cs598/note7.pdf}.

\bibitem[M{\"u}ller(1997)]{muller1997integral}
Alfred M{\"u}ller.
\newblock Integral probability metrics and their generating classes of
  functions.
\newblock \emph{Advances in Applied Probability}, 29\penalty0 (2):\penalty0
  429--443, 1997.

\bibitem[Mnih et~al.(2013)Mnih, Kavukcuoglu, Silver, Graves, Antonoglou,
  Wierstra, and Riedmiller]{mnih2013playing}
Volodymyr Mnih, Koray Kavukcuoglu, David Silver, Alex Graves, Ioannis
  Antonoglou, Daan Wierstra, and Martin Riedmiller.
\newblock Playing atari with deep reinforcement learning.
\newblock \emph{arXiv preprint arXiv:1312.5602}, 2013.

\bibitem[Wang(2017)]{wang2017primal}
Mengdi Wang.
\newblock Primal-dual $\pi$ learning: Sample complexity and sublinear run time
  for ergodic markov decision problems.
\newblock \emph{arXiv preprint arXiv:1710.06100}, 2017.

\end{thebibliography}
\bibliographystyle{unsrtnat}

\clearpage
\appendix


\newcommand{\bre}{\boldsymbol{\color{red}\EE_w[r]}}
\newcommand{\breh}{\boldsymbol{\color{red}\EE_{\hat w}[r]}}
\newcommand{\bbar}{\boldsymbol{\color{red}|}}
\newcommand{\bq}{\boldsymbol{\color{red}q(s_0, \pi)}}
\newcommand{\bqh}{\boldsymbol{\color{red}\hat{q}(s_0, \pi)}}

\begin{table}[t]
\caption{Comparison of the upper and lower bounds of $J(\pi)$ obtained by different OPE methods. Rows 1,2,4,5 are the bounds derived in this paper, and Rows 3,6 (with $\mathrm{'}$) are their na\"ive counterparts obtained by the approximation guarantees of previous methods, which are looser (see Appendix~\ref{app:naiveCI}). $\hat{w}$ (and $\hat{q}$) is the $w$ (and $q$) that attains the infimum in the rest of the expression, e.g., $\hat{w} := \argmin_{w}\sup_q|\Lmwl(w,q)|$. The \textbf{\color{red} red} terms highlight why the na\"ive bounds are loose. All expressions with subscript ``$\mathrm{w}$'' are valid upper or lower bounds when $\conv(\Qcal)$ is realizable, and those with ``$\mathrm{q}$'' are derived assuming realizable $\conv(\Wcal)$. \newline  $\Lmwl(w,q) := q(s_0, \pi) + \EE_w[\gamma q(s', \pi) - q(s,a)].$ \quad $\Lmql(w,q) := \EE_w[r + \gamma q(s', \pi) - q(s,a)].$ 
\newline Our loss $\Lwq = \EE_w[r] + \Lmwl(w,q) = q(s_0, \pi) + \Lmql(w,q).$  \label{tab:compare}}
\centering
\begin{tabular}{|c|c|c|}
	\hline
		& Expression & Remark \\ \hline
		$\ubw$ & $\inf_w \sup_q \Big( \bre + \Lmwl(w,q) \Big)$ &  New \\ \hline
		$\lbw$ & $\sup_w \inf_q \Big(\bre + \Lmwl(w,q) \Big)$ & Fenchel AlgaeDICE \citep{nachum2019algaedice}\\ \hline
		\tabincell{c}{$\ubw'$\\$\lbw'$} & $\breh \pm \inf_w \sup_q \bbar \Lmwl(w,q) \bbar$ & \tabincell{c}{MWL \citep{uehara2019minimax} \\ see also \citet{liu2018breaking, zhang2019gendice}} \\ \hline
		$\ubq$ & $\inf_q \sup_w \Big( \bq +\Lmql(w,q) \Big)$ & $=\lbw$ with convex $\Qcal$ and $\Wcal$ \\ \hline 
		$\lbq$ & $\sup_q \inf_w \Big(\bq + \Lmql(w,q) \Big)$ & $=\ubw$ with convex $\Qcal$ and $\Wcal$ \\ \hline
		\tabincell{c}{$\ubq'$\\$\lbq'$} & $\bqh \pm \inf_q \sup_w \bbar \Lmql(w,q) \bbar $ & MQL \citep{uehara2019minimax} \\ \hline
	\end{tabular}
\end{table}

\section{Proofs of Section~\ref{sec:ope}} 
\label{app:proof}

\begin{proof}[\textbf{Proof of Theorem~\ref{thm:ciw_valid}}]
	Let $w_1$ denote the $w$ that attains the infimum in $\ubw$.\footnote{Point-wise supremum is lower semi-continuous, i.e., $\supq \Lwq$ is lower semi-continuous in $w$ and hence the infimum is attainable.} 
	\begin{align*}
	\ubw - J(\pi) = &~ \supq \Lq{w_1} - L(w_1, Q^\pi) \tag{Lemma~\ref{lem:pel_q}: $\forall w$, $J(\pi) = \Lw{Q^\pi}$} \\
	\ge &~ \infw \supq \{\Lwq - \Lw{Q^\pi}\} \\
	= &~ \infw \supq \Lmwl(w, q-Q^\pi). \tag{$\EE_{w}[r]$ terms cancel}
	\end{align*}
	Similarly, let $w_2$ denote the $w$ that attains the supremum in $\lbw$,
	\begin{align*}
	J(\pi) - \lbw = &~ L(w_2, Q^\pi) - \infq \Lq{w_2}  
	\ge \infw \supq \{\Lw{Q^\pi} - \Lwq\}.
	\end{align*}
	The corollary follows because $\forall w$, both $\supq\{\Lwq - \Lw{Q^\pi}\}$ and $\supq \{\Lw{Q^\pi} - \Lwq\}$ are non-negative when $Q^\pi \in \conv(\Qcal)$, noting that $\Lw{\cdot}$ is affine.
\end{proof}

\begin{proof}[\textbf{Proof of Theorem~\ref{thm:ciw_tight}}]
	\begin{align*}
	\ubw - \lbw = &~ \infw \supq \Lwq - \supw \infq \Lwq \\
	= &~ \inf_{w, w'\in\Wcal} \left\{\supq \Lwq -  \infq \Lq{w'} \right\} \\
	\le &~ \infw \left\{\supq \Lwq -  \infq \Lwq \right\}. \tag{constraining $w=w'$}
	\end{align*}
	Noting that $\Lwq = \EE_w[r] + \Lmwl(w, q)$,  
	\begin{align*}
	\supq \Lwq -  \infq \Lwq
	= \supq \Lmwl(w,q) - \infq \Lmwl(w,q) \le 2\supq|\Lmwl(w,q)|. 
	\end{align*}
	When $\wpi\in \Wcal$, $\ubw\le \lbw$ follows from the fact that $\Lmwl(\wpi, q) \equiv 0$, $\forall q$.
\end{proof}

\begin{proof}[\textbf{Proof of Theorem~\ref{thm:ciq_valid}}]
	Let $q_1$ denote the $q$ that attains the infimum in $\ubq$. 
	\begin{align*}
	\ubq - J(\pi) = &~ \supw \Lw{q_1} - L(\wpi, q_1) \tag{$\forall q$, $J(\pi) = \Lq{\wpi}$} \\
	\ge &~ \infq \supw \{\Lwq - \Lq{\wpi}\}.
	\end{align*}
	Similarly, let $q_2$ denote the $q$ that attains the supremum in $\lbq$,
	\begin{align*}
	J(\pi) - \lbq = &~ L(\wpi, q_2) - \infw \Lw{q_2}  
	\ge \infq \supw \{\Lq{\wpi} - \Lwq\}.
	\end{align*}
	The corollary follows because $\forall q$, both $\supCw \{\Lwq - \Lq{\wpi}\}$ and $\supCw \{\Lq{\wpi} - \Lwq\}$ are non-negative when $\wpi \in \conv(\Wcal)$.
\end{proof}

\begin{proof}[\textbf{Proof of Theorem~\ref{thm:ciq_tight}}]
	\begin{align*}
	\ubq - \lbq = &~ \infq \supw \Lwq - \supq \infw \Lwq \\
	= &~ \inf_{q, q'\in\Qcal} \left\{\supw \Lwq -  \infw \Lw{q'} \right\} \\
	\le &~ \infq \left\{\supw \Lwq -  \infw \Lwq \right\}. \tag{constraining $q=q'$}
	\end{align*}
	Note that $\Lwq = q(s_0, \pi) + \Lmql(w, q)$, therefore, 
	\begin{align*}
	\supw \Lwq -  \infw \Lwq
	= \supw \Lmql(w,q) - \infw \Lmql(w,q) \le 2\supw|\Lmql(w,q)|. 
	\end{align*}
	When $\qpi\in \Qcal$, $\ubq\le \lbq$ follows from the fact that $\Lmql(w, \qpi) \equiv 0$, $\forall w$.
\end{proof}

\begin{proof}[\textbf{Proof of Theorem~\ref{thm:convex}}]
Since $\Lwq$ is bi-affine---which implies that it is convex-concave---by the minimax theorem \citep{neumann1928theorie, sion1958general}, the order of $\inf$ and $\sup$ are exchangeable. 
\end{proof}

\section{Relationships between the Intervals in the Non-convex Case} \label{app:nonconvex}
When $\Qcal$ and $\Wcal$ are non-convex, Theorem~\ref{thm:convex} no longer holds but the intervals derived in Sec.~\ref{sec:mwl-ci} and \ref{sec:mql-ci} are still related in an interesting albeit more subtle manner; as stated in Theorem~\ref{thm:nonconvex}, the two ``reversed intervals'' are generally tighter than the original intervals, but they are only valid under stronger realizability conditions. 

Below we provide the proof of Theorem~\ref{thm:nonconvex}.
\begin{proof}[\textbf{Proof of Theorem~\ref{thm:nonconvex}}]
	We only prove the first statement under $\qpi \in \Qcal$, and the proof for the second statement under $\wpi\in\Wcal$ is similar and omitted. First of all, $[\ubq, \lbq] \subseteq [\lbw, \ubw]$ holds without assuming $\qpi \in \Qcal$, as $\sup\inf(\cdot) \le \inf \sup(\cdot)$. Hence it suffices to show $J(\pi) \in [\ubq, \lbq]$. 
	
	To prove this statement, we go back to Lemma~\ref{lem:pel_q}, which states $J(\pi) = L(w, \qpi)$ holds for arbitrary $w$. Therefore,
	$$
	\sup_{w\in\Wcal} L(w, \qpi) = J(\pi) = \inf_{w\in\Wcal} L(w, \qpi). 
	$$
	Now if we have $\Qcal$ such that $\qpi\in\Qcal$, we immediately have
	$$
	\infq \supw L(w, q) \le \supw L(w, \qpi) = J(\pi) = \infw L(w, q) \le  \supq \infw L(w, q).
	$$
	This completes the proof. Compared to the derivation in Sec.~\ref{sec:mwl-ci}, we have changed the order in which $\Qcal$ and $\Wcal$ are introduced. As a consequence, we need to optimize $q\in\Qcal$ over the objectives $\supw L(w,q)$ and $\infw L(w,q)$ which are no longer affine in $q$ (due to $\supw$ and $\infw$). For this reason, $\qpi \in \conv(\Qcal)$ is no longer sufficient to guarantee the validity of the interval, and we need the stronger condition $\qpi\in\Qcal$. 
\end{proof}

In Theorem~\ref{thm:nonconvex} we provide tighter intervals with stronger realizability assumptions (e.g., $Q^\pi \in \Qcal$ instead of $Q^\pi \in \conv(\Qcal)$). Below we show that such assumptions are necessary, that when we only have $Q^\pi \in \conv(\Qcal)$ but not $Q^\pi \in \Qcal$, the tighter interval $[\ubq, \lbq]$ can be invalid. Similar conclusions hold for $[\ubw, \lbw]$ which we do not go over in detail. 

\begin{proposition}\label{prop:example_q_real}
There exists an MDP, a target policy $\pi$, a data distribution $\mu$, and function classes $\Qcal$ and $\Wcal$ satisfying $Q^\pi\notin\Qcal$ but $Q^\pi\in C(\Qcal)$, where $[\ubq, \lbq]$ is invalid.
\end{proposition}
\begin{proof}
We consider the deterministic MDP in Figure \ref{Fig:MDP_Example} with 3 states and 2 actions. $s_0$ is the only initial state, and $s_2$ is an absorbing state. $r(s,a)$ equals 1 if and only if $s=s_0$ and $a=a_0$, otherwise 0.
We consider uniform policy $\pi$, i.e. $\pi(s_i|a_j)=0.5$ for all $i\in\{0,1,2\},j\in\{0,1\}$. Let $\mu$ be the uniform distribution over $\Scal\times\Acal$.

\begin{figure*}[h!]
    \centering
	\includegraphics[scale=1.0]{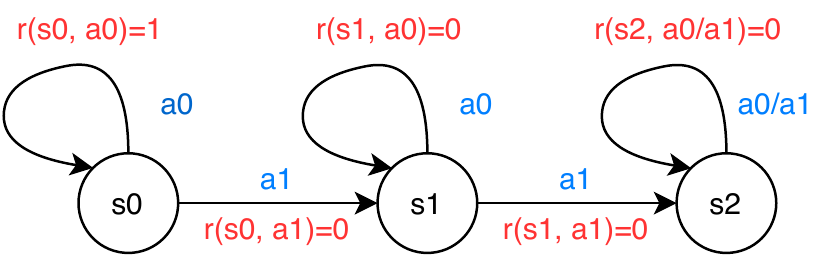}
	\caption{The MDP used in the proof of Proposition~\ref{prop:example_q_real}.}\label{Fig:MDP_Example}
\end{figure*}
We construct $\Qcal$ so that it contains two functions $q_1$ and $q_2$, defined as:
\begin{align*}
q_1(s,a)= Q^\pi(s,a)+ \epsilon \cdot \mathbb{I}[s=s_1], \qquad
q_2(s,a)=
    Q^\pi(s,a)-\epsilon \cdot \mathbb{I}[s=s_1], 
\end{align*}
for some $\epsilon>0$. It is easy to verify that $Q^\pi\notin \Qcal$ but $Q^\pi=\frac{1}{2}(q_1 + q_2) \in C(\Qcal)$. Also let $\Wcal$ be
\begin{align*}
    \mathcal{W}=\left\{\frac{1}{1-\gamma}\cdot\frac{f(s,a)}{\EE_\mu[f(s,a)]}: f \in [0, |\Scal\times\Acal|]^{\Scal\times\Acal}\right\}
\end{align*}
We pick out two elements $w_1,w_2$ from $\Wcal$: 
\begin{align*}
w_1(s,a)= \frac{6 \mathbb{I}[s=s_0,a=a_1]}{1-\gamma}, \qquad 
w_2(s,a)= \frac{6 \mathbb{I}[s=s_1,a=a_1]}{1-\gamma}.
\end{align*}
Let $
L(w,q):=q(s_0,\pi)+\EE_{w}[r+\gamma q(s',\pi)-q(s,a)]
$, and
\begin{align*}
    \sup_{w\in\Wcal} L(w,q_1)-J(\pi)=&\sup_{w\in\Wcal} L(w,q_1)-Q^\pi(s_0,\pi)\geq L(w_1,q_1)-Q^\pi(s_0,\pi)\\
    =&q_1(s_0,\pi)+\EE_{w_1}[r+\gamma q_1(s',\pi)-q_1(s,a)]-Q^\pi(s_0,\pi)\\
    =&\frac{r(s_0,a_1)+\gamma (\frac{1}{2}q_1(s_1,a_0)+\frac{1}{2}q_1(s_1,a_1))-q_1(s_0,a_1)}{1-\gamma}\\
    =&\frac{\gamma\epsilon}{1-\gamma}>0.\\
    \sup_{w\in\Wcal} L(w,q_2)-J(\pi)=&\sup_{w\in\Wcal} L(w,q_2)-Q^\pi(s_0,\pi)\geq L(w_2,q_2)-Q^\pi(s_0,\pi)\\
    =&q_2(s_0,\pi)+\EE_{w_2}[r+\gamma (q_2(s',\pi))-q_2(s,a)]-Q^\pi(s_0,\pi)\\
    =&\frac{r(s_1,a_1)+\gamma (\frac{1}{2}q_2(s_2,a_0)+\frac{1}{2}q_2(s_2,a_1))-q_2(s_1,a_1)}{1-\gamma}\\
    =&\frac{\epsilon}{1-\gamma}>0.
\end{align*}
Therefore,
$
\inf_{q\in\Qcal}\sup_{w\in\Wcal}L(w,q)-J(\pi)=\min\{\sup_{w\in\Wcal} L(w,q_1), \sup_{w\in\Wcal} L(w,q_2)\}>0,
$ 
which implies that the lower bound is invalid. The upper bound can be shown to be invalid in a similar manner.
\end{proof}

\section{On Double Robustness} \label{app:dr}
Our interval is valid when either function class is well-specified, which can be viewed as a type of double robustness. This is related to but different from the usual notion of double robustness in RL \citep{dudik2011doubly, jiang2016doubly, KallusNathan2019EBtC, tang2019harnessing}: classical doubly robust methods are typically ``meta''-estimators and require a value-function whose estimation procedure is unspecified, and the double robustness refers to the fact that the estimation is unbiased and/or enjoys reduced variance if the given value-function is accurate. In comparison, our double robustness gives weaker guarantees (valid interval, as opposed to accurate point estimates) but also requires much weaker assumptions (well-specified function \emph{class} as opposed to an accurate function), so it is important not to confuse the two types of double robustness.
\section{On AlgaeDICE's Notations} \label{app:algaedice}
In this section, we clarify the difference between the notations in AlgaeDICE \citep{nachum2019algaedice} and ours. 

We take the objective in their Eq.(15) as an example ($\max_\pi$ is dropped):
\begin{align*}
    &\min_{\nu:\Scal\times\Acal\rightarrow \mathbb{R}}\max_{\zeta:\Scal\times\Acal\rightarrow \mathbb{R}}(1-\gamma)\EE_{s_0\sim\mu_0,a_0\sim\pi(s_0)}[\nu(s_0,a_0)]\\
    &+\EE_{(s,a)\sim d^D,s'\sim T(s,a),a'\sim \pi(s')}[(\gamma \nu(s',a')+r(s,a)-\nu(s,a))\cdot\zeta(s,a)-\alpha\cdot f(\zeta)].
\end{align*}
As we can see, if we choose $\alpha=0$, the above is almost precisely our $\textrm{UB}_q$ ($=\textrm{LB}_w$ with convex classes): their $T$ is our $P$; their $\nu$ is our $q$; their $\zeta$ is our $w$; 
their $\nu(s_0, a_0)$ term corresponds to our $q(s_0, \pi)$ as we assume deterministic initial state w.l.o.g.~(see our Footnote~\ref{ft:initial}); 
they take expectation over the finite sample $d^D$ while we assume exact expectation over $\mu$ (from which $D$ can be sampled; see also Appendix \ref{app:staterr} for related discussions on generalization errors); finally, they derive the expression using fully expressive function classes $\Scal\times\Acal\to\RR$, where our derivation always uses restricted function classes $\Qcal$ and $\Wcal$.

Other than the above items, the only remaining difference is in the normalization convention: we define $J(\pi)$ and $d^\pi$ (and hence $\wpi$ and $w$) all in a unnormalized manner, whereas they take the normalized versions, which is why the expression still differs by a factor of $(1-\gamma)$.
\section{Comparison to Na\"ive Intervals \label{app:naiveCI}} 

We discuss in further details here why our upper and lower bounds are tighter than the na\"ive ones from previous works (see Table~\ref{tab:compare} and Remark~\ref{rem:naiveCI}). We compare $\ubq$ and $\ubq'$ as an example, and the situation for the other pairs of bounds are similar. 

Recall that the upper bound derived from MQL \citep{uehara2019minimax} is
\begin{equation}
\ubq' =  \hq(s_0,\pi)+\inf_{q\in\mathcal{Q}}\sup_{w\in\mathcal{W}}\Big|\EE_w[\Lmql(w,q)]\Big|,
\end{equation}
where $\Lmql(w,q) := \EE_w[r+\gamma q(s',\pi)-q(s,a)]$, and $\hq = \argmin_{q\in\mathcal{Q}}\supw \Big|\Lmql(w,q)\Big|$. In comparison, our bound is
\begin{equation}
\ubq = \infq\supw\Big(q(s_0,\pi)+\Lmql(w,q)\Big).
\end{equation}
Both bounds are valid upper bounds with realizable $\conv(\Qcal)$. Below we show that our upper bound is never higher than its na\"ive counterpart (this result does not require realizable $\Qcal$). 

\begin{proposition} $\ubq \le \ubq'$.
\end{proposition}

\begin{proof}
\begin{align}
\ubq = &~ \infq \Big(q(s_0,\pi)+\supw \Lmql(w,q)\Big) \nonumber\\
\le &~ \hq(s_0,\pi)+\supw \Lmql(w,\hq) \label{eq:tight1} \\
\le &~ \hq(s_0,\pi)+ \supw |\Lmql(w,\hq)| \label{eq:tight2} \\
= &~ \hq(s_0,\pi)+ \infq \supw |\Lmql(w,q)| = \ubq'. \tag*{\qedhere}
\end{align}
\end{proof}

\begin{remark}
As we can see, the tightness of $\ubq$ comes from two sources: (1) that we perform a unified optimization and put $q(s_0,\pi)$ inside $\infq$ (reflected in Eq.\eqref{eq:tight1}), and (2) that we do not need the absolute value in our objective (reflected in Eq.\eqref{eq:tight2}). 
On the other hand, if $\mathcal{W}$ is symmetric---that is, $-w\in \Wcal, \forall w\in\Wcal$---then we only enjoy the first kind of tightness.\footnote{This is because $\Lmql(w, q)$ is linear in $w$, and Eq.\eqref{eq:tight2} becomes an identity.}
\end{remark}

\begin{remark} $\ubq = \ubq'$ requires $\supw \Lmql(w, \hq) = \supw |\Lmql(w, \hq)|$ as a necessary condition. Similarly, one can show that $\lbq = \lbq'$ requires $-\infw \Lmql(w, \hq) = \supw|\Lmql(w, \hq)|$. Therefore, as long as
$$
-\supq\infw\Lmql(w,q) \neq \infq \supw \Lmql(w,q),
$$
at least one side of our interval will be strictly tighter than before.
\end{remark}


\section{Regularization} \label{app:reg}
Here we show how to introduce regularization into our intervals in a way similar to \citep{nachum2019algaedice}. 

\para{Derivation using Fenchel–Legendre Transformation} We exemplify how to adapt the derivation in Sec.~\ref{sec:mwl-ci} to obtain the regularized interval in Sec.~\ref{sec:mwl-ci}; the adaptation of Sec.~\ref{sec:mql-ci} is similar which we leave to the readers. Our derivation uses Fenchel-Legendre transformation in a way similar to \citep{nachum2019algaedice, nachum2020reinforcement}: we assume that $f: \RR \to \RR$ has a convex conjugate $f^*$ (i.e.,  $f^*(x^*) = \sup_x \{x\cdot x^* - f(x)\}$) that satisfies $f^*(0) = 0$. Below we show that by inserting such an $f^*$ into the derivation, we can obtain an interval that uses $f$ as the regularization function.
\begin{align*}
J(\pi) =&~ \qpi(s_0,\pi) + \EE_{s,a\sim \mu}\Big[ f^*\left(\EE_{r,s'|s,a}\Big[r + \gamma \qpi(s',\pi)-\qpi(s,a)\Big]\right)\Big] \tag{$\Tcal \qpi = \qpi$ and $f^*(0) = 0$}\\ 
=&~ \qpi(s_0,\pi) + \EE_{s,a\sim \mu}\Big[\sup_x \left(x\cdot\EE_{r,s'|s,a}\Big[r + \gamma \qpi(s',\pi)-\qpi(s,a)\Big] -f(x) \right)\Big]\\
\geq&~\qpi(s_0,\pi) + \EE_{s,a\sim \mu}\Big[w(s,a) \cdot \EE_{r,s'|s,a}\Big[r + \gamma \qpi(s',\pi)-\qpi(s,a)\Big]-f(w(s,a))\Big]\tag{Replace $x$ for each $(s,a)$ by $w(s,a)$ using an arbitrary $w\in\Wcal$} \\
= &~ \qpi(s_0,\pi) + \EE_{\mu}\Big[w(s,a) \cdot \Big(r + \gamma \qpi(s',\pi)-\qpi(s,a)\Big) - f(w(s,a))\Big].
\end{align*}
Although $f$ may be nonlinear, it is only applied to $w$ and the entire expression is still affine in $Q^\pi$, so we can replace $Q^\pi$ with $\infq$ as before: 
$$
J(\pi) \geq \infq q(s_0,\pi) +\EE_{\mu}\Big[w(s,a) \cdot \Big(r + \gamma \qpi(s',\pi)-\qpi(s,a)\Big) - f(w(s,a))\Big].
$$
Since it holds for arbitrary $w$, we take $\supw$ and obtain the lower bound:
\begin{align} \label{eq:lbw_reg}
J(\pi) \geq \supw \left(\infq L(w,q) - \EE_{\mu}[f(w(s,a))]\right).
\end{align}
Similarly, the upper bound (under $\qpi \in \conv(\Qcal)$) is
\begin{align}\label{eq:ubw_reg}
J(\pi) \leq \infw \left(\supq L(w,q) + \EE_{\mu}[f(w(s,a))]\right).
\end{align} 

\para{Properties of the Regularized Interval} From the above derivation, we see that the regularized interval is valid when $\qpi\in\conv(\Qcal)$, which is the same condition needed for the validity of the unregularized interval in Sec.~\ref{sec:mwl-ci}. One may naturally wonder what their relationship is, and the answer is very simple: adding any nontrivial regularization loosens the interval. 

To see this, first notice that Eq.\eqref{eq:lbw_reg} and \eqref{eq:ubw_reg} differ from $\lbw$ and $\ubw$ by a term $\EE_\mu[f(w(s,a))]$: such a term is subtracted from the lower bound and added to the upper bound. Now recall that our derivation crucially relies on $f^*(0)=0$. A direct consequence is that $f$ must be non-negative, as $f^*(0) = \inf_x f(x)$. Therefore, the regularization increases the upper bound and decreases the lower bound, which makes the interval looser. The tightest bound is obtained without any regularization, i.e., $\lbw$ and $\ubw$, which corresponds to $f \equiv 0$, whose convex conjugate is
\begin{equation*}
f^*(x^*)=\left\{
\begin{aligned}
0, &~~~ x = 0; \\
+\infty, &~~~ x\neq 0. \\
\end{aligned}
\right.
\end{equation*}

\section{OPE Experiments} \label{app:exp}

\begin{figure*}[t!]
	\centering
	\includegraphics[width=0.48\textwidth]{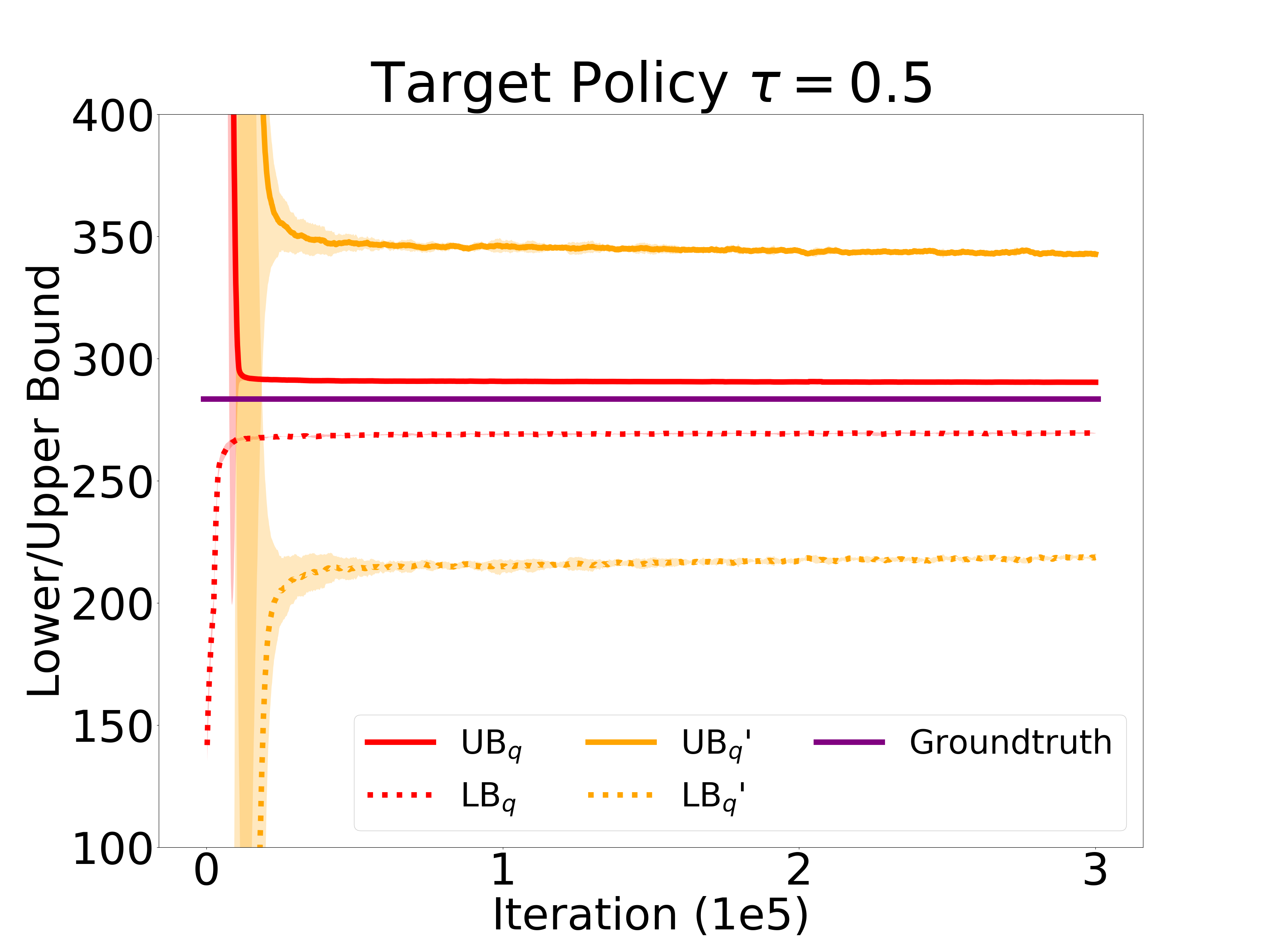} \hfill
	\includegraphics[width=0.48\textwidth]{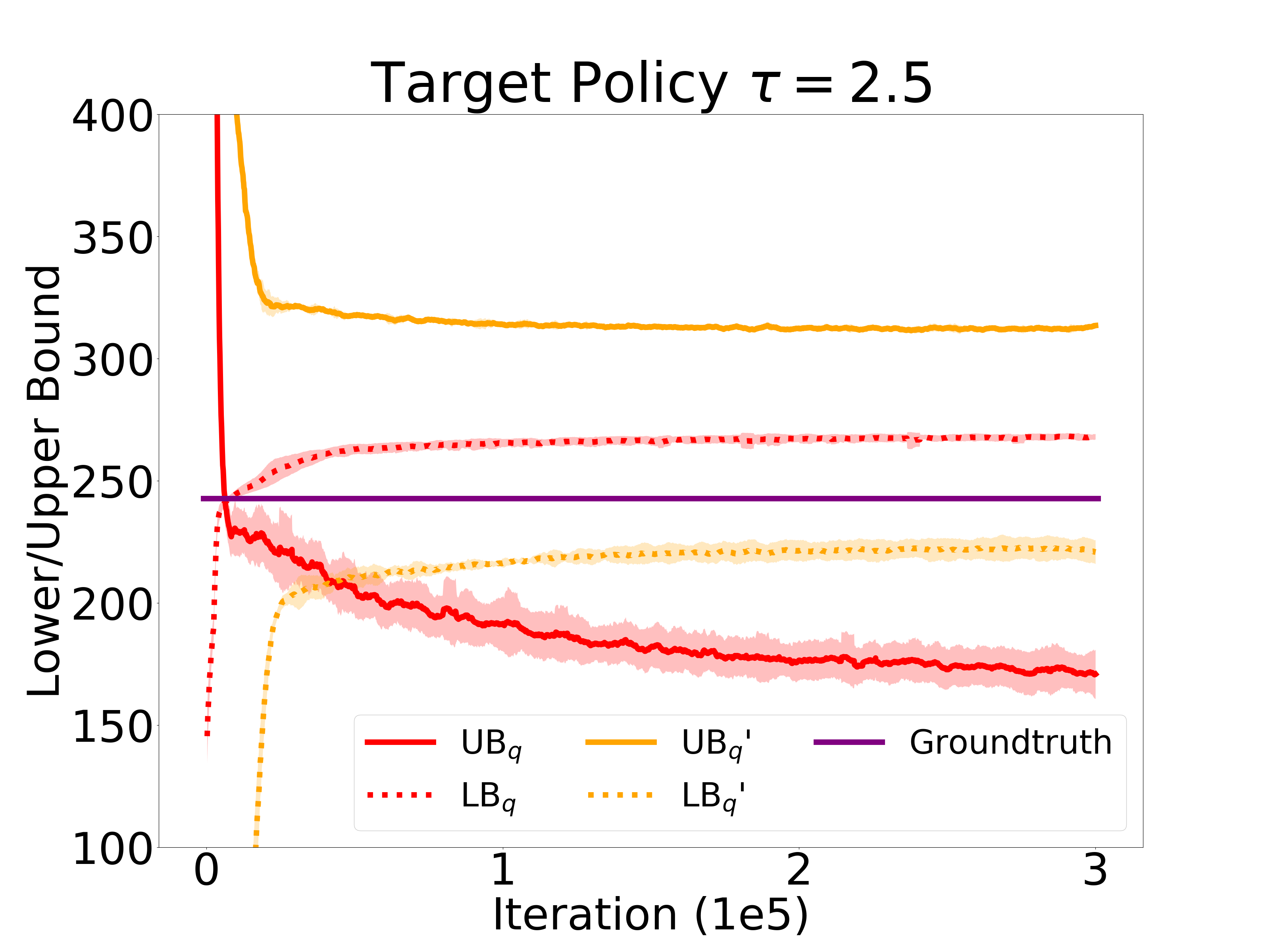}
	\caption{Example training curves averaged over 10 seeds. Error bars show twice the standard errors. The curves have been smoothed with a window size corresponding to 20 $q$-$w$ alternations. \textbf{Top}: The difference between behavior policy and target policy is relatively small, and $\ubq > \lbq$. The interval $[\lbq, \ubq]$ is significantly tighter than $[\lbq', \ubq']$. \textbf{Bottom}: Target policy is significantly different from the behavior policy, and realizing the importance weight is challenging. In this case, a reversed interval is observed, i.e. $\rm UB_q < LB_q$. \label{fig:learning_curve}}
\end{figure*}

\subsection{Environment and Behavior \& Target Policies}
We conduct experiments in the CartPole environment  with $\gamma=0.99$. Following \citet{uehara2019minimax}, we modify the reward function and add small Gaussian noise to transition dynamics to make OPE more challenging in this environment.\footnote{Many policies are indistinguishable under the original 0/1 reward function, so we define an angle-dependent reward function  that takes numerical values. We also add random noise to make the transitions stochastic.} 
To generate the behavior and the target policies, we apply softmax on a near-optimal $Q$-function trained via the open source code\footnote{\url{https://github.com/openai/baselines}} of DQN with an adjustable temperature parameter  $\tau$:
\begin{equation}
    \pi(a|s)\propto \exp(\frac{Q(s,a)}{\tau}).
\end{equation}
The behavior policy $\pi_b$ is chosen as $\tau=1.0$, and we use other values of $\tau$ for target policies. 
To collect the dataset, we truncate the generated trajectories at the 1000-th time step. For those terminated within 1000 steps, we pad the rest of the trajectories with the terminal states. We treat $\mu = d^\pi$ and approximate such a data distribution by weighting each data point $(s,a,r,s')$ with a weight $\gamma^t$, where $t$ is the time step $s$ is observed. All experiments generate datasets of $200$ trajectories except for the interval length comparison (Fig.~\ref{fig:len}), where the sample size is indicated on the x-axis. We report average results over 10 seeds in all the OPE experiments, and show twice the standard errors as error bars which correspond to $95\%$ confidence intervals.

\subsection{Details of the Algorithms}
We compare $\ubq$ and $\lbq$ to $\ubq'$ and $\lbq'$ in Table~\ref{tab:compare}. 
We use Multilayer Perceptron (MLP) to construct $\Qcal$ and $\Wcal$. The detailed specification of $\Qcal$ will be given later. For $\Wcal$, the ideal choice denoted by $\mathcal{W}^\alpha$ is defined as:
$$
\mathcal{W}^\alpha=\Big\{w(\cdot,\cdot)= \frac{\alpha |f(\cdot,\cdot)|}{\EE_{\mu}[|f|]}\Big|f\in {\rm MLP}\Big\}.
$$
Note that normalizing with $\EE_{\mu}[|f|]$ allows us to directly control the expectation of any $w\in\Wcal$ to be $\alpha$, and we use $\alpha = 25/(1-\gamma)$ throughout the OPE experiments. 
However, directly optimizing over $\mathcal{W}^\alpha$ is quite unstable, and we consider the following relaxation of our upper and lower bounds: fixing any $q\in\Qcal$, for any $w\in\Wcal^\alpha$, we relax the $\EE_w[r + \gamma q(s', \pi) - q(s,a)]$ term in $\ubq$ as
\begin{align} 
& \EE_{w}\Big[r+\gamma q(s',\pi)-q(s,a)\Big] \nonumber\\
= &\EE_{\mu}\Big[(r+\gamma q(s',\pi)-q(s,a))\frac{\alpha|f|}{\EE_{\mu}[|f|]}\Big] \nonumber\\
\le & \EE_{\mu}\Big[(r+\gamma q(s',\pi)-q(s,a))\frac{\alpha|f|\mathbb{I}[r+\gamma q(s',\pi)-q(s,a)>0]}{\EE_{\mu}[|f|]}\Big], \label{eq:opt_trick}
\end{align}
where $\mathbb{I}[\cdot]=1$ if the predicate is true, and $0$ otherwise. 
So essentially we turn each $f$ into a weighting function $w$ that evaluates to $\frac{\alpha|f|\mathbb{I}[r+\gamma q(s',\pi)-q(s,a)>0]}{\EE_{\mu}[|f|]}$ on each data point $(s,a,r,s')$. 
Such a relaxation helps stabilize training and results in  a looser upper bound compared to $\ubq$, which is still valid as long as $\ubq$ is a valid upper bound. We similarly relax the lower bound by replacing $\mathbb{I}[r+\gamma q(s',\pi)-q(s,a)>0]$ with $\mathbb{I}[r+\gamma q(s',\pi)-q(s,a)<0]$.

In addition, although optimizing MQL loss with $\mathcal{W}^\alpha$ can converge, we find that using the same relaxation can further stabilize training and lead to better results. Therefore, we adopt this trick in the calculation of $\rm{UB_q}'$ and $\rm{LB_q}'$ as well. 

We use a $32\times 32$ MLP (with tanh activation) to parameterize $q$ (except in Fig~\ref{fig:reverse} where the architecture is indicated on the x-axis) and use a one-hidden-layer MLP with 32 units for $f$ (which produces $w$). Moreover, we clip $q$ in the interval $[\frac{\Rmin}{1-\gamma} - \delta, \frac{\Rmax}{1-\gamma}]$, where $\Rmin=0.0$, $\Rmax=3.0$, $\delta=20$. 
We use stochastic gradient descent ascent (SGDA) with minimatches for optimization, and alternate between $q$ and $w$ every 500 and 50 iterations, respectively. The learning rates are both fixed as $0.005$, and each minibatch consists of $500$ transition tuples. The normalization factor in the relaxed objective is approximated on the minibatch. During the $q$-optimization phases, the weights $w$ (which depends on $q$ through the indicators) is treated as a constant and does not contribute to the gradients, but is re-computed every time $q$ is updated. See example training curves in Fig.~\ref{fig:learning_curve}.

\subsection{Bootstrapping Intervals} \label{app:bootstrap}
To account for statistical errors in our interval, we sample the dataset with replacement to generate 20 bootstrapped datasets with the same size as the original dataset. Then, we run our algorithms on those new datasets and record the upper/lower bounds. We pick the $k$-th largest (smallest) upper (lower) bound as the final upper (lower) bound, and for figures shown here we use $k=1$. We report interval lengths and validity ratios averaged over 10 runs. 

\subsection{On Comparison to MQL}
The comparison in Fig.~\ref{fig:len} should be interpreted carefully. This figure is used to demonstrate the theoretical predictions in App.~\ref{app:naiveCI}, where the two methods use the same $\Qcal$ and $\Wcal$ classes. Note that our choice of $\Wcal$ class in this experiment, $\Wcal^\alpha$, comes with a hyperparameter $\alpha$ that adjusts the magnitude of the functions in the class, and we choose a fairly large value of $\alpha = 25/(1-\gamma)$ for stability of optimization. (A more principled optimization approach is also an interesting future direction, which may allow us to choose a much smaller value of $\alpha$ for our intervals.) The loss of MQL, on the other hand, is homogeneous in $w\in\Wcal$ hence the algorithm is \emph{invariant} to the rescaling of $\Wcal$. Therefore, the value of $\alpha$ has no effect on the training process and merely determines the length of the interval in a straightforward manner (i.e., when we change $\alpha$, MQL's interval length scales linearly with $\alpha$, and its center does not move). To this end, the $\alpha$ for MQL could have been tuned to obtain a tighter yet still valid interval, although this is difficult in practice as we do not have the groundtruth value to tune $\alpha$ against (otherwise OPE would not be necessary; see discussions in \citep{voloshin2019empirical}). Nevertheless, we reiterate that Fig.~\ref{fig:len} is only meant to empirically illustrate the theoretical predictions of App.~\ref{app:naiveCI}, and to compare the two methods fairly the value of $\alpha$ should be tuned for MQL in some fashion, which we do not investigate in this paper.



\section{Sanity-Check Experiments} \label{app:tabular_exp}
In App.~\ref{app:exp} we introduced several tricks to stabilize the difficult minimax optimization problems associated with our intervals. While the tricks have worked out empirically in CartPole, we would like to further understand the legitimacy and the consequences of these tricks. 

The major modification is the relaxation in Eq.\eqref{eq:opt_trick}, where we allow $w\in\Wcal$ to depend not only on $(s,a)$ but also on the randomness of $(r,s')$. A potential caveat is that, if the transition or reward function are stochastic, given two tuples $(s_1,a_1,r_1,s'_1)$ and $(s_2,a_2,r_2,s'_2)$, it is possible that $w(s_1,a_1)\neq w(s_2,a_2)$ if $s_1=s_2,a_1=a_2$ but $r_1\neq r_2$ or $s'_1\neq s'_2$. In comparison, without the relaxation in Eq.\eqref{eq:opt_trick}, we shall always have $w(s_1,a_1)=w(s_2,a_2)$. While we may hardly find two identical state-action pairs in continuous tasks, the issue still exists if we observe states that are considered close to each other by the function approximator $\Wcal$. Therefore, we hypothesize that this relaxation may make the interval loose when the environment is highly stochastic, and an ideal way to test this hypothesis is to conduct an experiment in a tabular environment because (1) it is easy to find tabular environments with sufficient stochasticity, and (2) we can afford to implement a more faithful version of our algorithm in the tabular setting, which could be unstable and even diverge in the function-approximation setting.

\paragraph{Experiment Setup} 
Given the above considerations, we conduct additional experiments in Taxi, 
a tabular environment with 2000 states and 6 actions. 
We  parameterize $Q$ with $|\mathcal{S}|\times|\mathcal{A}|$ bounded variables, i.e., $\Qcal = \left[\frac{R_{\min}}{1-\gamma}, \frac{R_{\max}}{1-\gamma}\right]^{\Scal\times\Acal}$.
As for $\mathcal{W}$, we consider three different function classes:
\begin{align*}
    \mathcal{W}^\alpha_1=&\{w(s,a)=\frac{\alpha f(s, a)\mathbb{I}[r+\gamma Q(s',\pi)-Q(s,a)>0]}{\EE_{(\tilde{s},\tilde{a},\tilde{r},\tilde{s}')\sim \mu}[f(\tilde{s}, \tilde{a})]},\forall (s,a)\in\mathcal{S}\times\mathcal{A}\}\\
    &\cup \{w(s,a)=\frac{\alpha f(s, a)\mathbb{I}[r+\gamma Q(s',\pi)-Q(s,a)<0] }{\EE_{(\tilde{s},\tilde{a},\tilde{r},\tilde{s}')\sim \mu}[f(\tilde{s}, \tilde{a})]}
    ,\forall (s,a)\in\mathcal{S}\times\mathcal{A}\}\\
    \mathcal{W}^\alpha_2=&\{w(s,a)=\frac{\alpha f(s, a)\mathbb{I}[\EE_{r,s'|s,a}[r+\gamma Q(s',\pi)-Q(s,a)]>0] }{\EE_{(\tilde{s},\tilde{a})\sim \mu}[f(\tilde{s}, \tilde{a})]},\forall (s,a)\in\mathcal{S}\times\mathcal{A}\}\\
    &\cup \{w(s,a)=\frac{\alpha f(s, a)\mathbb{I}[\EE_{r,s'|s,a}[r+\gamma Q(s',\pi)-Q(s,a)]<0] }{\EE_{(\tilde{s},\tilde{a})\sim \mu}[f(\tilde{s}, \tilde{a})]},\forall (s,a)\in\mathcal{S}\times\mathcal{A}\}\\
    \mathcal{W}^\alpha_3 =& \{w(s,a)=\frac{\alpha f(s, a)}{\EE_{(\tilde{s},\tilde{a})\sim \mu}[f(\tilde{s}, \tilde{a})]},\forall (s,a)\in\mathcal{S}\times\mathcal{A}\}\\
\end{align*}
where $f$ can take independent values in different $(s,a)$ (i.e., we use a tabular representation for $f$). We only constrain $f$ to be positive but do not clip its value; during the training processes its value is usually bounded.

As we can see, $\mathcal{W}^\alpha_1$ precisely corresponds to the relaxation in Eq.\eqref{eq:opt_trick}, where we replaced the MLP $f$ with a tabular $f$. We compare it to two alternatives: $\Wcal_2^\alpha$ and $\Wcal_3^\alpha$: in $\mathcal{W}^\alpha_2$, we first integrate out the randomness in $\EE_{r,s'|s,a}$ so that $w$ does not depend on $(r,s')$, but still keep the modification related to the use of indicator function $\mathbb{I}[\ldots > 0]$ that avoids negativity. This indicator trick is further removed in $\Wcal_3^\alpha$, where we have a completely faithful implementation of the original algorithm. When optimizing with $\Wcal_1^\alpha$ and $\Wcal_2^\alpha$, we adopt SGDA and alternate between optimizing $Q$ and $w$ every 250 iterations. The learning rate for both $Q$ and $w$ are 5e-3, and the batch size is 500. When optimizing with $\mathcal{W}^\alpha_3$, 
we found that SGDA can lead to unstability, and we instead update $Q$ and $w$ synchronously at each training iteration. (As a side note, synchronous updates fail in CartPole experiments.)  We use the whole dataset in each iteration, and the learning rate is fixed to 5e-3.

The results are showed in Figure \ref{fig:sanity_check}. As we can see, the dependence of $w$ on $(r,s')$ indeed result in a looser bound asymptotically. Moreover, the use of the indicator trick does not make the interval loose in this case, while it converges more slowly than synchronously updating $Q$ and $w$ with $\mathcal{W}^\alpha_3$. The looseness of $\Wcal_1^\alpha$ suggests that our optimization strategy in Appendix~\ref{app:exp} is not general enough, and it will be important to investigate more principled optimization strategies that directly work with the original optimization problem without the relaxation in Eq.\eqref{eq:opt_trick}. 

\begin{figure*}[t!]
	\centering
	\includegraphics[scale=0.3]{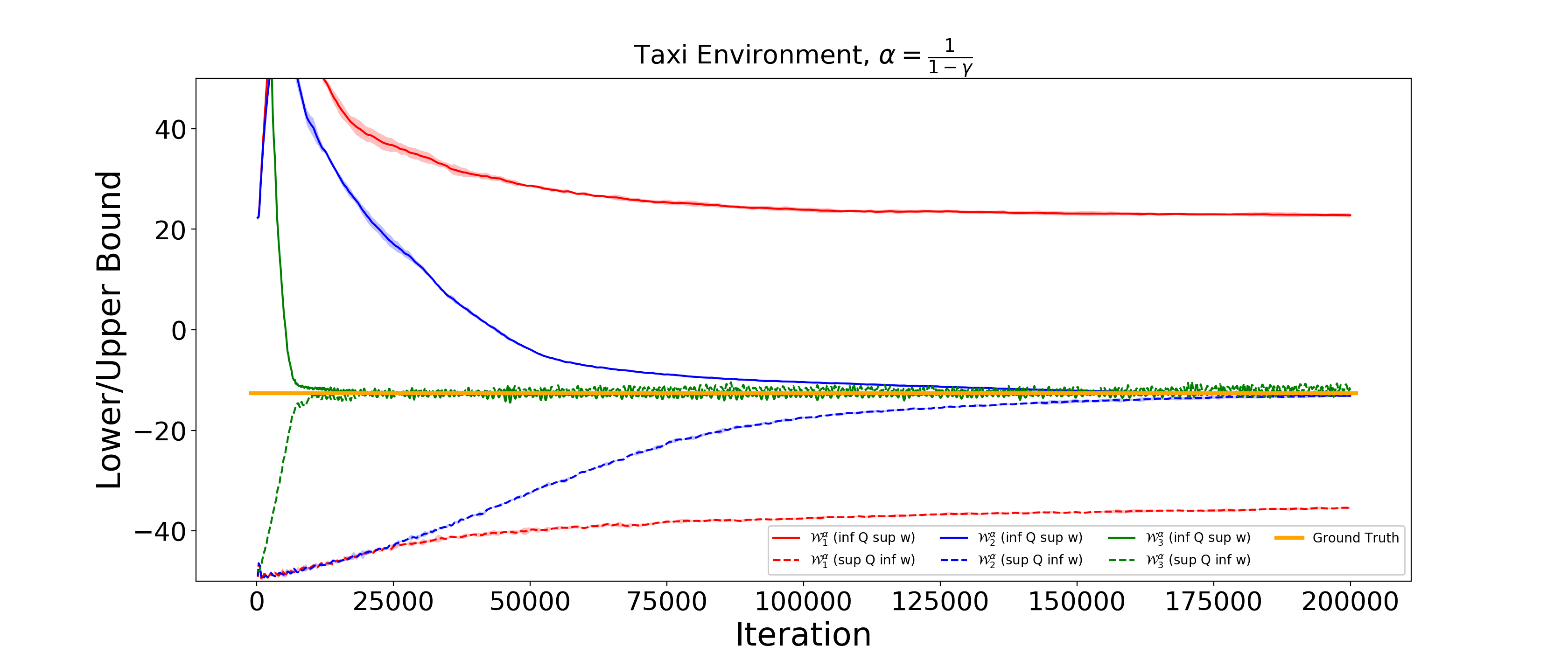}
	\caption{Comparison between experiments on Taxi with different function classes. Average over 3 seeds. \textbf{Red Curves}: Upper/Lowers bound  using $\mathcal{W}^\alpha_1$. \textbf{Blue Curves}: Upper/Lower bounds  using $\mathcal{W}^\alpha_2$. \textbf{Green Curves}: Upper/Lower bounds  using $\mathcal{W}^\alpha_3$.}\label{fig:sanity_check}
\end{figure*}


\section{Rmax / Rmin} 
\label{app:rmax}
Here we review the concept of Rmax and Rmin in detail, and provide the proof of Proposition~\ref{prop:Rmin}. We first recall the setup of Sec.~\ref{sec:tabular}:

\para{Setup} Consider an MDP with finite and discrete state and action spaces. Let $\Sk \subset \Scal$ be a subset of the state space, and $s_0 \in \Sk$. Suppose the dataset contains $n$ transition samples from each $s \in \Sk, a\in\Acal$ with $n \to \infty$ (i.e., the transitions and rewards in $\Sk$ are fully known), and $0$ samples from $s \notin \Sk$.

\para{Rmax and Rmin} In such a simplified setting, effective exploration can be achieved by (the episodic version of) the  well-known Rmax algorithm \citep{brafman2003r,kakade2003sample}, which computes the optimal policy of the Rmax-MDP $\MRmax$, defined as $(\Scal, \Acal, P_{\MRmax}, R_{\MRmin}, \gamma, s_0)$, where\footnote{With a slight abuse of notations, in this section we treat $R: \Scal\times\Acal\to[0, \Rmax]$ as a deterministic reward function for convenience.}
\begin{align*}
	P_{\MRmax}(s,a) = \begin{cases}
		P(s,a), & \textrm{if~} s \in \Sk\\
		\Indi[s'=s], & \textrm{if~} s \notin \Sk
	\end{cases}, \qquad
	R_{\MRmax}(s,a) = \begin{cases}
		R(s,a), & \textrm{if~} s \in \Sk\\
		\Rmax, & \textrm{if~} s \notin \Sk
	\end{cases}.
\end{align*}
In words, $\MRmax$ is the same as the true MDP $M$ on states with sufficient data, and the remaining ``unknown'' states are assumed to have self-loops with $\Rmax$ immediate rewards, making them appealing to visit and thus encouraging exploration. The theoretical guarantee for the policy computed by Rmax is an optimal-or-explore statement \citep[see e.g.,][]{nan_rmax_notes}, that either the optimal policy of $\MRmax$ is near-optimal in true $M$, or its occupancy measure must visit states outside $\Sk$, with a mass proportional to its suboptimality. 
Similarly, when the goal is to exploit, that is, to output a policy with the best worst-case guarantee, one simply needs to change the $\Rmax$ in the Rmax-MDP to the minimally possible reward value (which is $0$ in our setting), and we call this MDP $\MRmin$.

\begin{remark}
	In the simplified setting above, the states outside $\Sk$ receive no data at all. In the more general case where $\Sk$ is still under-explored but every state receives any least $1$ data point, it is not difficult to show that both MUB-PO and MLB-PO reduce to the certainty-equivalent solution, and can overfit to the poor estimation of transitions and rewards in states with few data points. Such a degenerate behavior may be prevented by regularizing $w$ and prohibiting any highly spiked $w$ (e.g., regularizing with $\EE_{\mu}[w^2]$, which measures the effective sample size of importance weighted estimate induced by $w$), or by bootstrapping and taking data randomness into consideration during policy optimization. 
\end{remark}

\begin{proof}[\textbf{Proof of Proposition~\ref{prop:Rmin}}]
	Let $\Mcal$ be the space of all MDPs that are consistent with the given dataset. Let $\MRmax$ and $\MRmin$ be the Rmax and Rmin MDPs, resp., and it is clear that $\MRmax, \MRmin \in \Mcal$. 
	
	\paragraph{MLB-PO Reduces to Rmin} To show that MLB-PO reduces to Rmin, it suffices to prove that for every $\pi: \Scal\to\Acal$, $\lbwpi = J_{\MRmin}(\pi)$. Since $\Qcal$ is the tabular function space and always realizable (regardless of the true MDP $M$), $\lbwpi$ is a valid lower bound, i.e., $\lbwpi \le J(\pi)$. Since all MDPs in $\Mcal$ \emph{could} be the true $M$, it must hold that $\lbwpi \le J_{\MRmin}(\pi)$. It suffices to show that $\lbwpi \ge J_{\MRmin}(\pi)$. 
	
	Recall that
	$$
	\lbwpi = \supw \infq  \left\{ q(s_0, \pi) + \EE_w[r+ \gamma q(s', \pi) - q(s,a)]\right\}.
	$$
	Since $\Qcal$ is the unrestricted tabular function space, we may represent $q\in\Qcal$ as a $|\Scal\times\Acal|$ vector where each coordinate $q(s,a)$ can take values between $[0, \Rmax/(1-\gamma)]$ independently. Now note that for any $s'\notin \Sk$, $a' \in \Acal$, $q(s', a')$ as a decision variable may only appear as the $\gamma q(s',\pi)$ term in the objective and can never appear as $q(s_0, \pi)$ or $-q(s,a)$. Since $\Wcal$ only contains non-negative functions, $\Lwq$ is non-decreasing in $q(s', \pi)$, and $\infq$ can always be attained when $q(s',a') = 0$, $\forall s'\notin\Sk, a'\in\Acal$. For any $w\in\Wcal$, let $q_w$ denote such a $q\in\Qcal$ that achieves the inner infimum, and
	$$
	\lbwpi = \sup_{w\in\Wcal} \{q_w(s_0,\pi) + \EE_w[r + \gamma q_w(s', \pi) - q_w(s,a)]\}.
	$$
	Next, consider the discounted occupancy of $\pi$ in $\MRmin$, denoted as $d_{\MRmin}^\pi$. Define
	$$
	w_0(s,a) := \Indi[s\in\Sk] \frac{d_{\MRmin}^\pi(s,a)}{1/|\Sk\times\Acal|}.
	$$  
	Let $q_0 = q_{w_0}$. 
	Now, since $w_0 \in \Wcal$, 
	\begin{align*}
		\lbwpi \ge &~ q_0(s_0, \pi) + \EE_{w_0}[r + \gamma q_0(s', \pi) - q_0(s,a)] \\
		= &~ q_0(s_0, \pi) + \frac{1}{|\Sk\times\Acal|}\sum_{s \in \Sk, a\in\Acal} w_0(s,a) (R(s,a) + \gamma \EE_{s'\sim P(s,a)}[q_0(s',\pi)] - q_0(s,a)) \\
		= &~ q_0(s_0, \pi) + \sum_{s \in \Sk, a\in\Acal} d_{\MRmin}^\pi(s,a) (R(s,a) + \gamma \EE_{s'\sim P(s,a)}[q_0(s',\pi)] - q_0(s,a)) \\
		= &~ \sum_{s \in \Sk, a\in\Acal} d_{\MRmin}^\pi(s,a) R(s,a) + \sum_{s\in\Scal, a\in\Acal} \nu(s,a) q_0(s,a),
	\end{align*}
	where 
	$$
	\nu(s,a):= \Indi[s=s_0, a= \pi(s_0)] + \gamma \sum_{s'\in\Sk, a'\in\Acal} d_{\MRmin}^\pi(s',a') P(s|s',a')- \Indi[s \in \Sk] d_{\MRmin}^\pi(s,a).
	$$
	Note that 
	$$
	\sum_{s\in\Sk, a\in\Acal} d^{\pi}_{\MRmin}(s,a) R(s,a) = \sum_{s\in\Sk, a\in\Acal} d^{\pi}_{\MRmin}(s,a) R(s,a) + \sum_{s\notin\Sk, a\in\Acal} d^{\pi}_{\MRmin}(s,a) \cdot 0 = J_{\MRmin}(\pi),
	$$ 
	so it suffices to show that $\sum_{s\in\Scal, a\in\Acal} \nu(s,a) q_0(s,a) = 0$, which we establish in the rest of this proof.
	
	Recall that $q_0(s,a)=0$ for any $s\notin \Sk$. So $\sum_{s\in\Scal, a\in\Acal}\nu(s,a) q_0(s,a) = \sum_{s\in\Scal, a\in\Acal}\nu'(s,a) q_0(s,a)$, where
	$$
	\nu'(s,a):= \Indi[s=s_0, a= \pi(s_0)] + \gamma\sum_{s'\in\Scal, a'\in\Acal} d_{\MRmin}^\pi(s',a') P_{\MRmin}(s|s',a')- d_{\MRmin}^\pi(s,a).
	$$
	This is because $\nu$ and $\nu'$ exactly agree on the value in $s\in\Sk$, and only differ on $s\notin \Sk$. To see this, consider any $s\in\Sk, a\in\Acal$. $\nu$ and $\nu'$ agree on the first and the last terms. They also agree on the second term for the summation variables $s'\in\Sk, a'\in\Acal$, as $P(s|s',a') = P_{\MRmin}(s|s',a')$ (the MDP $\MRmin$ has the true transition probabilities on states with sufficient data). So the only difference is 
	$$
	\gamma\sum_{s'\notin \Sk, a'\in\Acal} d_{\MRmin}^\pi(s',a') P_{\MRmin}(s|s',a').
	$$
	However, this term must be zero, because for any $s\in \Sk$, $s'\notin\Sk$, $P_{\MRmin}(s|s',a') = 0$, as the construction of $\MRmin$ guarantees that no states outside $\Sk$ will ever transition back to $\Sk$. Now that we conclude $\nu(s,a) q_0(s,a) = \nu'(s,a) q_0(s,a)$, the fact that it is zero is obvious: $\nu' \equiv 0$, which follows directly from the Bellman equation for discounted occupancy in MDP $\MRmin$. This completes the proof of $\lbwpi = J_{\MRmin}(\pi)$. 
	
	\paragraph{MUB-PO Reduces to Rmax} The proof for $\ubwpi = J_{\MRmax}(\pi)$ is similar. 
	Using the same argument as above, we know that  $\ubwpi \ge J_{\MRmax}(\pi)$ and it suffices to show that $\ubwpi \le J_{\MRmax}(\pi)$.
	
	Recall that
	$$
	\ubwpi = \infw \supq  \left\{ q(s_0, \pi) + \EE_w[r+ \gamma q(s', \pi) - q(s,a)]\right\}.
	$$
	Next, we consider the discounted occupancy of $\pi$ in $\MRmax$, denoted as $d^\pi_{\MRmax}$. Define $w_0$ as (we recycle the symbol from the previous part of the proof)
	\begin{align*}
		w_0(s,a):=\Indi[s\in\Sk] \frac{d_{\MRmax}^\pi(s,a)}{1/|\Sk\times\Acal|}.
	\end{align*}
	and $q_0$ as
	\begin{align*}
		q_0 := \arg\max_{q\in\Qcal} \left\{ q(s_0, \pi) + \EE_{w_0}[r+ \gamma q(s', \pi) - q(s,a)]\right\}.
	\end{align*}
	With the same argument as the Rmin case, we can always have $q_0(s,a)=\frac{\Rmax}{1-\gamma},\forall s\notin\Sk,a\in\Acal$. 
	
	Since $w_0\in\Wcal$,
	\begin{align*}
		\ubwpi \leq&~q_0(s_0, \pi) + \EE_{w_0}[r+ \gamma q_0(s', \pi) - q_0(s,a)]\\
		=&~q_0(s_0, \pi) + \frac{1}{|\Sk\times\Acal|}\sum_{s \in \Sk, a\in\Acal} w_0(s,a) (R(s,a) + \gamma \EE_{s'\sim P(s,a)}[q_0(s',\pi)] - q_0(s,a)) \\
		=&~q_0(s_0, \pi) + \sum_{s \in \Sk, a\in\Acal} d_{\MRmax}^\pi(s,a) (R(s,a) + \gamma \EE_{s'\sim P(s,a)}[q_0(s',\pi)] - q_0(s,a)) \\
		=&~ q_0(s_0, \pi) + J_{\MRmax}(\pi) - \sum_{s \notin \Sk, a\in\Acal} d_{\MRmax}^\pi(s,a) \Rmax \\
		&~+ \sum_{s \in \Sk, a\in\Acal} d_{\MRmax}^\pi(s,a) (\gamma \EE_{s'\sim P(s,a)}[q_0(s',\pi)]-q_0(s,a)).
	\end{align*}
	The last step follows from the fact that in $\MRmax$, all $s\not\in\Sk$ are absorbing states with $\Rmax$ rewards. 
	Now that we have extracted out $J_{\MRmax}(\pi)$, it suffices to show that the remaining terms cancel. 
	
	To show the cancellation, we will use two properties of the $\MRmax$ construction: (1) For any $s\notin\Sk$, $q_0(s,a) - \gamma \EE_{s'\sim P_{\MRmax}(s,a)}[q_0(s', \pi)] = \frac{\Rmax}{1-\gamma} - \frac{\gamma\Rmax}{1-\gamma} = \Rmax$ (because $P_{\MRmax}(s,a)$ is a point mass on $s$), and (2) for any $s\in\Sk$, $P_{\MRmax}(s,a) = P(s,a)$. Using these, we may rewrite the remaining terms as
	\begin{align*}
		& q_0(s_0, \pi) - \sum_{s \notin \Sk, a\in\Acal} d_{\MRmax}^\pi(s,a) \Rmax +\sum_{s \in \Sk, a\in\Acal} d_{\MRmax}^\pi(s,a) (\gamma \EE_{s'\sim P(s,a)}[q_0(s',\pi)]-q_0(s,a)) \\
		=& q_0(s_0, \pi) + \sum_{s \in \Scal, a\in\Acal}d_{\MRmax}^\pi(s,a) (\gamma \EE_{s'\sim P_{\MRmax}(s,a)}[q_0(s',\pi)]-q_0(s,a)).
	\end{align*}
	This is $0$ due to the Bellman equation for occupancy $d_{\MRmax}^\pi$ in the MDP $\MRmax$ (which we have also used in Lemma~\ref{lem:pel_w}).
\end{proof}
\section{Additional Proofs, Results, and Discussions of Section~\ref{sec:opt}}  \label{app:mubpo}
\subsection{Proof of Proosition~\ref{prop:mlbpo}}
It suffices to show that $J(\hat \pi) \ge \mathrm{LB_w^{\hat\pi}}$, which follows directly from Theorem~\ref{thm:ciw_valid}: $\lbwpi$ is a valid lower bound for any $\pi \in \Pi$ due to the realizability of $\Qcal$. The corollary holds because for $\pi$ with $\wpi$ realized by $\Wcal$, Theorems~\ref{thm:ciw_valid} and \ref{thm:ciw_tight} guarantee that $J(\pi) = \lbwpi$. \qed

\subsection{Full Version of Proposition~\ref{prop:mubpo} and Proof}

\begin{proposition}[Full version of Proposition~\ref{prop:mubpo}] \label{prop:mubpo_full}
	Let $\Pi$ be a policy class. Let $\hat \pi = \argmax_{\pi\in\Pi} \ubwpi$.  Then, for any $w\in \Wcal$,
	$$
	\mathrm{IPM}(w\cdot \mu, d^{\hat \pi}; \Fcal) \ge \max_{\pi \in \Pi: Q^\pi\in\conv(\Qcal)} J(\pi) - J(\hat \pi),
	$$
	where 
	\begin{itemize}
		\item $(w\cdot \mu)(s,a):= w(s,a) \cdot \mu(s,a)$, 
		\item $\mathrm{IPM}(\nu_1, \nu_2; \Fcal):= \sup_{f\in\Fcal} |\EE_{\nu_1}[f] - \EE_{\nu_2}[f]|$, 
		\item $\Fcal := \{\Tcal^\pi q - q: q\in\Qcal, \pi\in\Pi\}$. 
	\end{itemize}
	As a corollary, if we further assume $\|q\|_\infty \le \Rmax/(1-\gamma), \forall q\in\Qcal$,
	$$
	\|w\cdot \mu - d^{\hat \pi}\|_1 \ge \frac{(1-\gamma) \left(\max_{\pi \in \Pi: Q^\pi\in\conv(\Qcal)} J(\pi) - J(\hat \pi)\right)}{2\Rmax}.
	$$
\end{proposition}

To interpret the result, recall that $w\in\Wcal$ is supposed to model an importance weight function that coverts the data distribution to the occupancy measure of some policy, e.g., $\wpi \,\cdot\, \mu = d^\pi$. The proposition states that either $\hat \pi$ is near-optimal, or it will induce an occupancy measure that cannot be accurately modeled by \emph{any} importance weights in $\Wcal$ when applied on the current data distribution $\mu$. The distance\footnote{$w \cdot \mu$ may be unnormalized, but this does not affect our results.} between the two distributions $d^\pi$ and $w \cdot \mu$ is measured by the Integral Probability Metric \citep{muller1997integral} defined w.r.t.~a discriminator class $\Fcal$, and can be relaxed to the looser but simpler $\ell_1$ distance. Therefore, if we have a rich $\Wcal$ class that models all distributions covered by $\mu$, then $\hat{\pi}$ must visit new areas in the state-action space or it must be near-optimal. 

Below we give the proof of Proposition~\ref{prop:mubpo_full}, which reuses many results established in Sec.~\ref{sec:ope}.

\begin{proof}[Proof of Proposition~\ref{prop:mubpo_full}]
	Fixing any $\pi\in\Pi$ such that $Q^\pi \in \conv(\Qcal)$:
	\begin{align*}
		J(\pi) - J(\hat \pi)
		\le &~ \mathrm{UB_w^{\pi}} - J(\hat \pi) \tag{$\ubw$ is valid upper bound as $Q^\pi$ realized by  $\conv(\Qcal)$} \\
		\le &~ \mathrm{UB_w^{\hat \pi}} - J(\hat \pi)  \tag{$\hat \pi$ optimizes $\ubwpi$}.
	\end{align*} 
	Recall that $\mathrm{UB_w^{\hat\pi}} = \infw \supq L(w,q;\hat\pi)$, so for any $w\in \Wcal$, $\mathrm{UB_w^{\hat\pi}} \le \supq L(w, q; \hat\pi)$. On the other hand, for any $q$, $J(\hat\pi) = \Lq{w_{\hat\pi/\mu}}$. Now let $q_w:= \argmax_{q\in\Qcal} L(w,q;\hat\pi)$. For any $w\in\Wcal$, 
	\begin{align*}
		\max_{\pi \in \Pi: Q^\pi \in \Qcal} J(\pi) - J(\hat \pi)  
		\le &~  \mathrm{UB_w^{\hat \pi}} - J(\hat \pi) \\
		\le &~ L(w, q_w; \hat \pi) - L(w_{\hat\pi/\mu}, q_w; \hat\pi) \\
		= &~ \EE_{w}[r + \gamma q_w(s', \hat\pi) - q_w(s, a)] - \EE_{w_{\hat\pi/\mu}}[r + \gamma q_w(s', \hat\pi) - q_w(s, a)] \\
		\le &~ |\EE_w[\Tcal^{\hat \pi} q_w - q_w]] - \EE_{w_{\hat\pi/\mu}}[\Tcal^{\hat \pi} q_w - q_w]|.
	\end{align*} 
	The main statement immediately follows by noticing that $\hat{\pi} \in \Pi, q_w\in\Qcal$, hence $\Tcal^\pi q_w -q_w \in \Fcal$. The $\ell_1$-distance corollary follows from relaxing IPM using H\"older's inequality for the $\ell_1$ and $\ell_\infty$ pair. 
\end{proof}

\subsection{Alternative Guarantee for MLB-PO} \label{app:mlb-po-alt}
Proposition~\ref{prop:mlbpo} shows that MLB-PO puts the heavy expressivity burden on $\Qcal$ and is agnostic against misspecified $\Wcal$. When data has sufficient coverage and $\Wcal$ is highly expressive, we can similarly show that optimizing $\lbqpi$ performs robust exploitation and is agnostic against misspecified $\Qcal$.

\begin{proposition}[Exploitation with expressive $\Wcal$] \label{prop:mubpo-q}
	Let $\Pi$ be a policy class, and assume $\wpi \in \conv(\Wcal) ~ \forall \pi \in \Pi$. Let $\hat \pi = \argmax_{\pi\in\Pi} \lbqpi$. Then, for any $\pi\in\Pi$,
	$$
	J(\hat \pi) \ge \lbqpi. 
	$$
	As a corollary, for any $\pi$ such that $\qpi \in \Qcal$, we have $J(\hat \pi) \ge J(\pi)$, that is, we compete with any policy whose value function can be realized by $\Qcal$.
\end{proposition}
	The proof is similar to that of Proposition~\ref{prop:mlbpo} and hence omitted.

\subsection{Connection between MUB-PO and OLIVE \citep{jiang2017contextual}} \label{app:olive}
A complete algorithm for  exploration usually involves multiple iterations data collection and policy re-computation, and we only show that MUB-PO performs one such iteration effectively. There are further design choices needed to complete MUB-PO into a full algorithm. 
For example, one may repeatedly collect new data using the policy computed by MUB-PO and merge it with the data from previous rounds. However, it is difficult to analyze the algorithm theoretically: the realizability of $\conv(\Wcal)$ depends on $\mu$, but $\mu$ itself dynamically changes over the execution of the algorithm due to data pooling, and any realizability-type assumptions such as $\wpi \in \Wcal$ are no longer static and cannot appear in an \emph{a priori} guarantee. 
	
One interesting way to avoid this difficulty is to use a special class of importance weights $\Wcal$: instead of $w$ that depends on $(s,a)$, consider $w$ that is $(s,a)$-independent and is indicator function of the \emph{identity of the dataset}. That is, the $\sup_{w \in \Wcal}$ in $\ubwpi$ chooses among the datasets collected in different rounds, and takes expectation w.r.t.~only one of them (without reweighting within the dataset), essentially avoiding data pooling. The resulting algorithm is very similar to a parameter-free variant\footnote{The original OLIVER algorithm requires the approximation error of $\Qcal$ as an input, which can be avoided by replacing its constrained optimization step with an unconstrained one similar to MUB-PO.} 
of the OLIVER algorithm by \citet[Algorithm 3]{jiang2017contextual}, which has been shown to enjoy low-sample complexities in a wide range of low-rank environments.
\section{Policy Optimization Experiments} \label{app:opt_exp}

\begin{figure*}[t]
	\centering
	\includegraphics[width=0.48\textwidth]{Pictures/Tau0_1_discount.png}
	\includegraphics[width=0.48\textwidth]{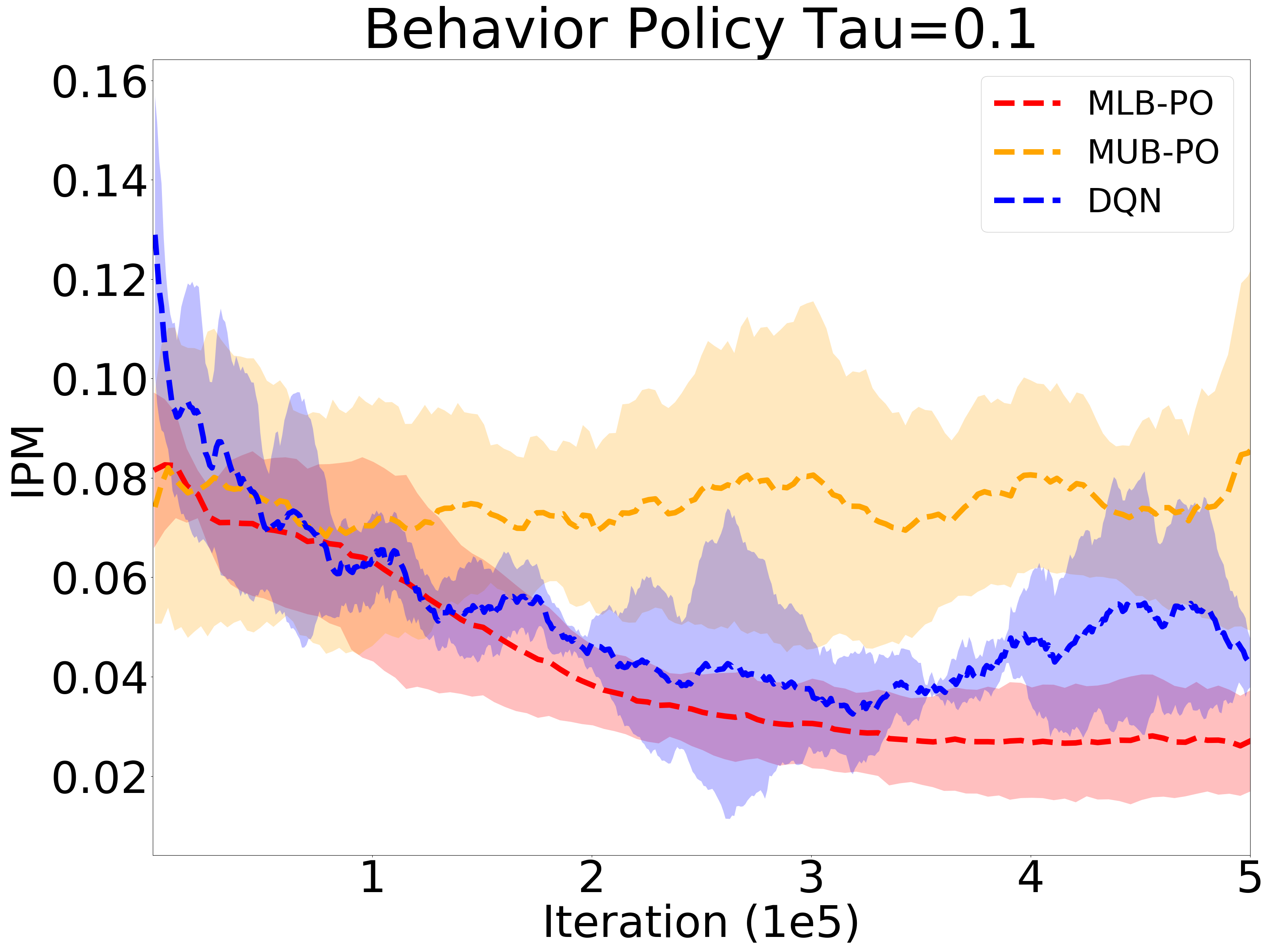}
	\includegraphics[width=0.48\textwidth]{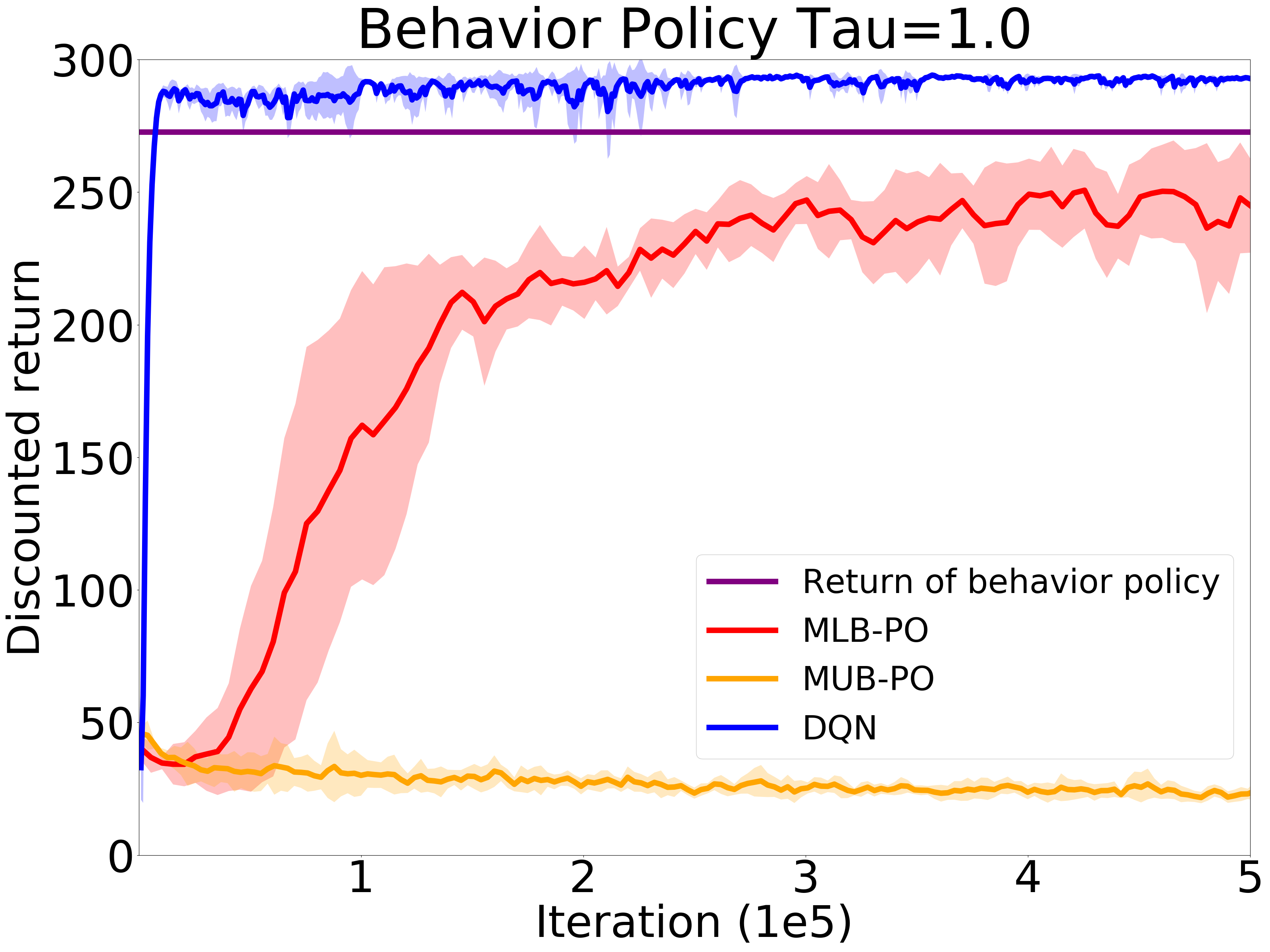}
	\includegraphics[width=0.48\textwidth]{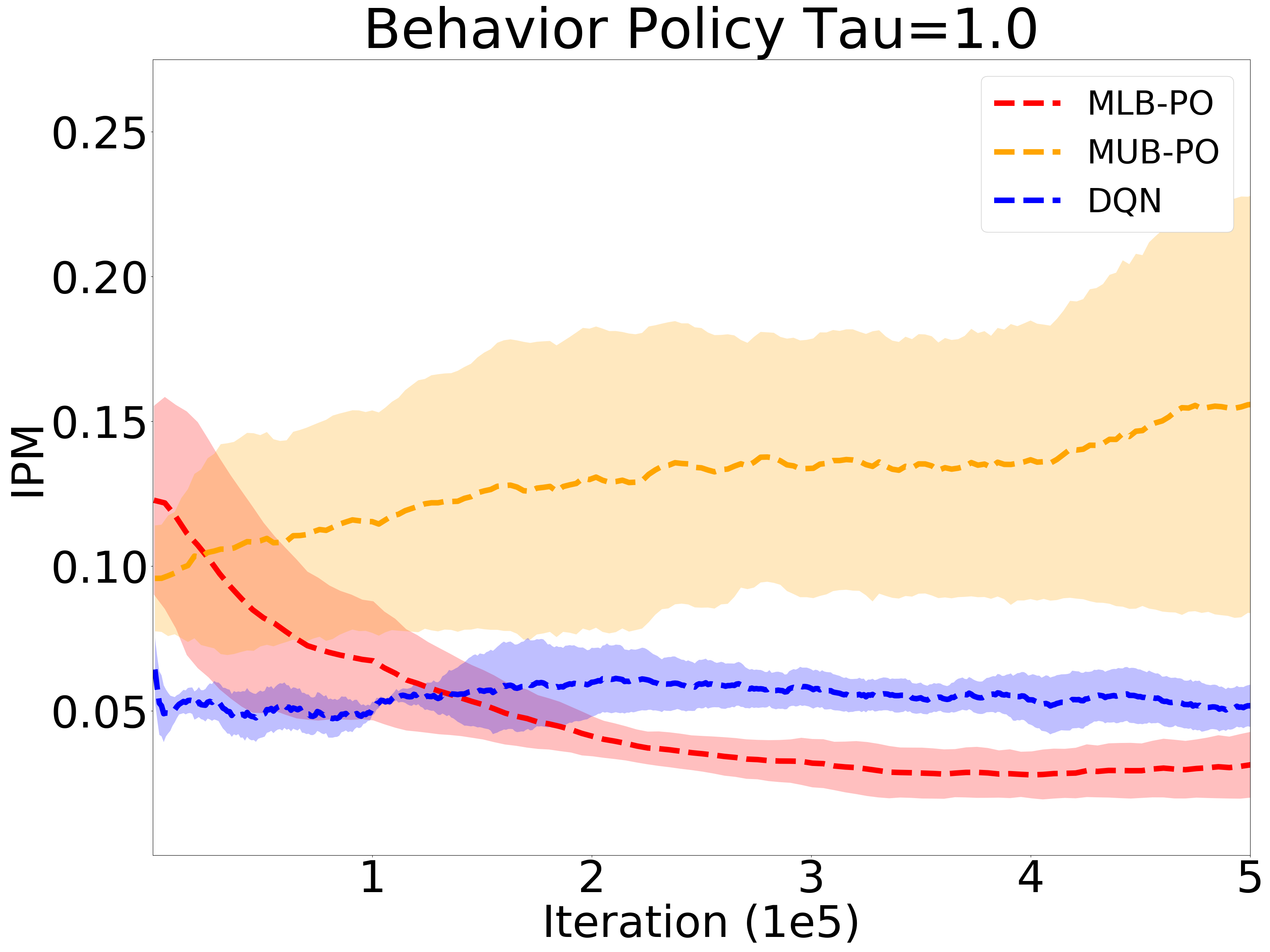}
	\caption{Expected return and IPM of the learned policies as a function of training iterations in the policy optimization experiments. Results in this figure are averaged over 5 random seeds. \textbf{Top row}: Behavior policy is $\tau=0.1$; \textbf{Bottom row}: Behavior policy is $\tau=1.0$.}\label{Fig:val_discount}
\end{figure*}

In this section, we report the policy optimization results in CartPole environment. Due to the difficulty of optimization, we follow \citet{nachum2019algaedice} and consider the following simplified heuristic version of MUB-PO and MLB-PO, which mostly applies in near-deterministic environments: 
\begin{align*}
    \argmax_\pi \max_q q(s_0,\pi) - \beta\EE_\mu[|r+\gamma q(s',\pi)-q(s,a)|] \tag{Simplified MUB-PO}\\
    \argmax_\pi \min_q q(s_0,\pi) + \beta\EE_\mu[|r+\gamma q(s',\pi)-q(s,a)|] \tag{Simplified MLB-PO}
\end{align*}
Here $\beta$ is a hyperparameter (similar to the role of $\alpha$ in the OPE experiments) that controls the level of pessimism in MLB-PO and optimism in MUB-PO. We parameterize the policy $\pi(\cdot|s)$ using a 32$\times$32 MLP with a softmax layer.

For MUB-PO, we update $q$ and $\pi$ iteratively every 500 iterations. As for the lower bound, in order to avoid the policy becoming greedy so quickly, we only update $\pi$ once before alternating to $Q$. 
Besides, we also report the results of DQN \cite{mnih2013playing} for comparison. 
Among all the experiments, we fix the learning rate as 5e$-3$ and the batch size as 500. All the Q-functions are parameterized by a 32$\times$32 MLP.

Moreover, we also record the IPM of each algorithms during training, defined as:
$$
\text{IPM} = \max_{f\in\mathcal{F}}\Big|\EE_{(s,a)\sim d^{\pi_b}}[f(s,a)]-\EE_{(s,a)\sim d^{\pi}}[f(s,a)]\Big|.
$$
The IPM score measures the difference of the state-action visitation between the behavior policy and the policy we are optimizing, which is inspired by Proposition~\ref{prop:mubpo_full} and reflects whether $\pi$ explores new states and actions outside the coverage of $d^{\pi_b}$. 
In practice, $\mathcal{F}$ is constructed by 50 randomly initialized 32$\times$32 MLPs.

In Fig.~\ref{Fig:val_discount}, we report the results with different behavior policies: $\tau=0.1$ (behavior policy is nearly deterministic) and $\tau=1.0$ (behavior policy is more exploratory). For MLB-PO and MUB-PO, we use $\beta=1/(1-\gamma)$ and $\beta=3/(1-\gamma)$ in $\tau=0.1$ and $\tau=1.0$, respectively. In both cases, MLB-PO delivers good performance as expected, and when data is non-exploratory ($\tau=0.1$) it outperforms DQN. Notably, MLB-PO does not use explicit regularization towards behavior policy as many popular algorithms do. MUB-PO, on the other hand, fails to obtain a policy of high return, but the IPM scores show that it can induce an occupancy significantly different from $d_{\pi_b}$, which is what our theory anticipated (Proposition~\ref{prop:mubpo_full}). 
\section{Handling Statistical Errors based on Generalization Error Bounds}  \label{app:staterr}
In this section we briefly explain how to construct a confidence bound that is valid with high probability in the presence of sampling errors due to only having access to a finite dataset. We emphasize that the construction based on generalization error bounds is typically loose for practical purposes and this section only serves as a conceptual illustration. We will use $\ubw:=\inf_{w\in\Wcal}\sup_{q\in\Qcal} L(w,q)$ as an example and the other upper/lower bounds can treated similarly. Let $\rm{\widehat{UB}_w}$ denote the sample-based approximation of $\ubw$:
\begin{align*}
    {\rm \widehat{UB}_w}=\inf_{w\in\Wcal}\sup_{q\in\Qcal}\hLwq:=q(s_0, \pi) + \EE_{D}[w(s,a)(r+\gamma q(s',\pi) - q(s,a))].
\end{align*}
where $D$ is a dataset with $n$ i.i.d.~data points $(s,a,r,s')$ sampled as $(s,a) \sim \mu, r \sim R(s,a), s' \sim P(s,a)$. We assume $\Wcal$ and $\Qcal$ have bounded ranges of output values, that is, there exists two constants $C_{\Wcal},C_{\Qcal}>0$, such that, $w\in\Wcal$, $q\in\Qcal$, $\|w\|_\infty \le C_{\Wcal}$, $\|q\|_{\infty} \le C_{\Qcal}$. 
As a result,
\begin{align*}
    |L(w,q)|\leq&|q(s_0, \pi)| + \EE_{\mu}[|w(s,a)|(r+\gamma |q(s',\pi)| + |q(s,a)|)]\\
    \leq& C_{\Qcal} + \frac{C_{\Wcal}}{1-\gamma}(1+(1+\gamma) C_{\Qcal}):=L_{\max}.
\end{align*}
In the following, we will use $L_{\max}$ as a shorthand for the upper bound of $|L(w,q)|$. 

According to standard results for Rademacher complexity, we have with probability $1-\delta$, for any $w\in\Wcal$ and $q\in\Qcal$, 
\begin{align}\label{eq:Rademacher_Complexity}
    \Lwq\leq \hLwq + 2\hRcal_D(\Wcal, \Qcal)+6L_{\max}\sqrt{\frac{\log\frac{2}{\delta}}{2n}}.
\end{align}
where  $\hRcal_D(\Wcal,\Qcal):=\EE_{\sigma}[\sup_{(w,q)\in\Wcal\times\Qcal}\frac{1}{n}\sum_{(s,a,r,s')\in D}\sigma_i L^{(s,a,r,s')}(w,q)]$ is the empirical Rademacher complexity of $\Wcal\times\Qcal$ w.r.t. sample $D$, which can be computed from data. 

We define $(w^*, q^*)=\arg\min_{w\in\Wcal}\max_{q\in\Qcal}\Lwq$ and $(\hat{w}^*, \hat{q}^*)=\arg\min_{w\in\Wcal}\max_{q\in\Qcal}\hLwq$ and use $q_{\hat{w}^*}$ to denote $\arg\max_{q\in\Qcal} \Lwq(\hat{w}^*, q)$. Then, we have:
\begin{align}
    {\rm UB_w}=&L(w^*, q^*)\leq L(\hat{w}^*, q_{\hat{w}^*}) \\
    \leq& \hat{L}(\hat{w}^*, q_{\hat{w}^*})+2\hRcal_D(\Wcal, \Qcal)+6L_{\max}\sqrt{\frac{\log\frac{2}{\delta}}{2n}}\nonumber\\
    \leq&\hat{L}(\hat{w}^*, \hat{q}^*)+2\hRcal_D(\Wcal, \Qcal)+6L_{\max}\sqrt{\frac{\log\frac{2}{\delta}}{2n}}\nonumber\\
    =&{\rm \widehat{UB}_w}+2\hRcal_D(\Wcal, \Qcal)+6L_{\max}\sqrt{\frac{\log\frac{2}{\delta}}{2n}}\label{eq:rade_UB_w}.
\end{align}
Eq.\eqref{eq:rade_UB_w} implies that, for a fixed policy $\pi$, with probability $1-\delta$, when ${\rm UB_w}$ is a valid upper bound (for example, when $Q^\pi$ is realizable), then we can guarantee that ${\rm \widehat{UB}_w}+2\hRcal_D(\Wcal, \Qcal)+6L_{\max}\sqrt{\frac{\log\frac{2}{\delta}}{2n}}$ is also a valid upper bound. 

\section{Variants of Minimax Value Interval: Using Knowledge of Behavior Policy and Finite-horizon / Average-reward Case} \label{app:variants}
While we derive the intervals in infinite-horizon discounted MDPs under the behavior-agnostic setting in the main paper, our derivations can be easily extended to other settings by replacing Lemmas~\ref{lem:pel_q} and \ref{lem:pel_w} with their counterparts and following the same recipe. Below we briefly introduce these extensions. 

\subsection{Using Knowledge of the Behavior Policy}
When the behavior policy that produces the data is known (e.g., $\pi_b$ such that $\mu(s,a) = \mu(s) \pi_b(a|s)$), one can derive similar intervals that use importance weighting over actions $\rho :=\frac{\pi(a|s)}{\pi_b(a|s)}$ with the following lemmas in place of Lemmas~\ref{lem:pel_q} and \ref{lem:pel_w}:
\begin{lemma}\label{lem:per_v}
For any $w:\Scal \to \RR$, 
\begin{align} \label{eq:pel_v}
J(\pi) -  \EE_w[\rho r] = V^\pi(s_0) + \EE_w[ \gamma \rho V^\pi(s') - V^\pi(s)].
\end{align}
\end{lemma}

\begin{lemma}\label{lem:per_sw}
For any $v: \Scal \to \RR$,
\begin{align} \label{eq:pel_sw}
J(\pi) -  v(s_0) = \EE_{d^\pi}[r + \gamma v(s') - v(s)].
\end{align}
\end{lemma}
Following the recipe in Sec.~\ref{sec:ope} and \ref{sec:opt}, one can prove the properties of the corresponding intervals that are analogues to the results in this paper. Note that we can also use these lemmas to derive the counterparts of MWL and MQL for state-functions, the former of which is a variant of \citet{liu2018breaking}'s method; see \citet[Appendix A.5 and Footnote 15]{uehara2019minimax}.

\subsection{Finite-horizon Case}
Finite-horizon MDPs can be modeled as infinite-horizon discounted MDPs with initial state distributions (which is our setting) by including time-step as a part of the state representation, modeling termination as absorbing states, and setting $\gamma = 1$. 

\subsection{Average-reward Case}
The average-reward case is somewhat more different from the discounted or the finite-horizon setting, so we go over it in more details. Suppose the Markov chain induced by $\pi$ is ergodic and let $d^\pi$ be its stationary state-action distribution (note that $d^\pi$ here is normalized, unlike the discounted occupancy in the discounted setting). Similarly we define $\wpi(s,a) := \frac{d^\pi(s,a)}{\mu(s,a)}$, and the quantity of interest is 
\begin{align}
J(\pi) = \lim_{T\to \infty} \frac{1}{T}\sum_{t=0}^{T-1} r_t = \EE_{d^\pi}[r] = \EE_w[r]. 
\end{align}
As before, we need two Bellman equations, one for the distribution and the other one for the value function. The one for distribution is simple: for any $q: \Scal\times\Acal\to \RR$,
\begin{align}
\EE_{d^\pi}[q(s,a) - q(s', \pi)] = 0.
\end{align}
Recall that $\EE_{d^\pi}[\cdot]$ means taking expectation over $(s,a) \sim d^\pi, r \sim R(s,a), s' \sim P(s,a)$. This equation is  just a re-expression of the fact that $d^\pi$ is the stationary distribution induced by $\pi$.

The Bellman equation for the value function is: for any $w: \Scal\times\Acal\to\RR$,
\begin{align}
\EE_w[Q^\pi(s,a) + J(\pi) - r - Q^\pi(s',\pi)] = 0.
\end{align}
Since there is no initial state distribution in the average-reward case, even if $Q^\pi$ is learned, one cannot plug-in initial state distribution to obtain the estimation of $J(\pi)$. Instead, $J(\pi)$ itself directly appears as a separate quantity in the Bellman equation. In fact, value functions in the average-reward case are actually ``difference of values'', which are analogues to the advantage function in the discounted case and can be informally defined as
$$
Q^\pi(s,a) = \EE\left[\sum_{t=0}^\infty (r_t - J(\pi)) ~\middle|~ s_0 = s, a_0 = a, a_{1:\infty} \sim \pi\right]. 
$$
According to this definition, the ``temporal difference'' $Q^\pi(s,a) - r - Q^\pi(s', \pi)$ is not zero in expectation, and a $J(\pi)$ term remains, allowing $J(\pi)$ to be directly manifested in the Bellman equations. 

With the above background, we are ready to state the counterparts of Lemmas~\ref{lem:pel_q} and \ref{lem:pel_w} for the average-reward case:
\begin{lemma}\label{lem:per_avq}
	For any $w:\Scal \times \Acal \to \RR$, 
	\begin{align} \label{eq:pel_avq}
	J(\pi) = \EE_w[ r + Q^\pi(s', \pi) - Q^\pi(s,a)].
	\end{align}
\end{lemma}

\begin{lemma}\label{lem:per_avw}
	For any $q: \Scal \times \Acal \to \RR$,
	\begin{align} \label{eq:pel_avw}
		J(\pi) = \EE_{\wpi}[r + q(s', \pi) - q(s,a)].
	\end{align}
\end{lemma}

The resulting intervals are $\{\inf_w \sup_q, \sup_w \inf_q, \inf_q \sup_w, \sup_q\inf_w\} L(w,q)$, where $L(w,q) := \EE_w[r + q(s', \pi) - q(s,a)]$. We note that the loss $L(w,q)$ is highly similar to the primal-dual form of LP for MDPs \citep{wang2017primal}.
\end{document}